%% file: colm2026_conference.tex
\documentclass{article} 
\usepackage[final]{colm2026_conference}

\usepackage{microtype}
\usepackage{hyperref}
\usepackage{url}
\usepackage{booktabs}
\usepackage{enumitem}
\usepackage{amsmath,amsfonts,amssymb,amsthm}
\usepackage{graphicx}
\usepackage{xcolor}
\input{math_commands.tex}

\usepackage{lineno}

\definecolor{darkblue}{rgb}{0, 0, 0.5}
\hypersetup{colorlinks=true, citecolor=darkblue, linkcolor=darkblue, urlcolor=darkblue}

\title{Unveiling the Mechanisms of Multi-Hop Reasoning in Transformers via Identity Bridge}

\author{Pengxiao Lin$^{1,2,}$\thanks{Equal contribution.} \quad
Zheng-An Chen$^{1,*}$ \quad
Zhi-Qin John Xu$^{1,2,3}$ \thanks{Corresponding author: \texttt{xuzhiqin@sjtu.edu.cn}.}\\
$^{1}$School of Mathematical Sciences, Shanghai Jiao Tong University \\
$^{2}$Institute of Natural Sciences, MOE-LSC, Shanghai Jiao Tong University \\
$^{3}$Shanghai Seres Information Technology Co., Ltd, Shanghai 200040, China.
}
\newtheorem{corollary}{Corollary}[section]

\begin{document}

\ifcolmsubmission
\linenumbers
\fi

\maketitle

\begin{abstract}
Large Language Models (LLMs) excel at multi-hop reasoning in distribution, yet fail on unseen compositions, a phenomenon known as the ``curse of two-hop reasoning". In this work, we argue that this phenomenon can be attributed to a missing supervision on the bridge entity. We formalize this gap by introducing identity bridge, a minimal supervision that enforces a identity mapping on bridge tokens. Under this supervision, even a one-layer transformer with uniform attention (Emb–MLP) can achieve out-of-distribution (OOD) two-hop generalization. We provide a theoretical analysis demonstrating that identity bridge induces an implicit regularization effect, leading the model to establish a direct subject-to-answer association. From an empirical perspective, the performance of standard GPT-2 models aligns closely with simple Emb--MLP models across varying levels of problem complexity. Finally, analyses of fine-tuned mainstream LLMs indicate that correct two-hop predictions consistently coincide with the establishment of a subject-to-answer relationship, extending our findings to realistic settings.
\end{abstract}

\section{Introduction}\label{sec::Intro}

Large language models (LLMs) achieve strong performance on a wide range of multi-step reasoning tasks, especially when assisted by chain-of-thought (CoT) \citet{wei2022cot, ToT, step_by_step}. Despite their strong CoT-assisted performance, LLMs exhibit very fragile reasoning abilities when forced to produce the final answer directly \citet{bai2025canttransformerslearnmultiplication}. On the one hand, models sometimes exhibit signatures consistent with latent multi-step processing, yet such evidence is suggestive rather than definitive and is often difficult to disentangle from shortcut strategies \citet{shortcut1, shortcut2}. On the other hand, controlled synthetic settings reveal striking and systematic failures, including the reversal curse \citet{berglund2024the, allen-zhu2025physics_3_1, allen-zhu2025physics_3_2, qi2023investigation} and the two-hop curse \citet{balesni2024two, 2023Faith}. These phenomena suggest that seemingly minor design choices in synthetic data and training objectives can crucially interact with optimization-induced inductive biases, leading to brittle multi-hop inference. While recent work has offered partial analyses and empirical observations \citet{2024Do, hoptwolate, 2024Grokked, 2025How, liu2026layer}, a theoretically grounded explanation of mechanisms behind multi-hop reasoning remains incomplete. This motivates the following questions.

\begin{enumerate}[leftmargin=*]
\item Why is OOD compositional generalization in two-hop reasoning so brittle, and what overlooked factors cause this brittleness in a way that admits rigorous characterization?

\item When a transformer appears to generalize in two-hop settings, what mechanism underlies this behavior: a genuine two-step implicit inference, or a shortcut solution that mimics composition?
\end{enumerate}

In this work, we revisit the two-hop curse, where models fit two single-hop relations ($a \to b$ and $b \to c$ in the training set) yet fail to answer the composed query ($a \to c$ at evaluation), following the setup and definitions of \citet{2024Grokked}. We highlight an often overlooked factor that standard synthetic setups do not enforce the bridge entity to be identical in the input and the output. Without this constraint, training can fit the two single-hop tasks using disjoint features, causing composition to break down. We introduce identity bridge, a minimal fix that adds a zero-hop self-mapping for each bridge token, unlocking OOD two-hop generalization. This also clarifies a gap between synthetic two-hop settings and pretrained models. Pretraining effectively endows models with a latent form of identity bridging, such as the ability to copy or restate text, which can partially support composition.

To better understand the mechanism, we analyze a simplified Emb--MLP model, namely a one-layer transformer with uniform attention that has been widely used to study transformer reasoning \citet{yao2025analysisreasoningbiaslanguage, huang2025generalization, 1layertransformer1, 1layertransformer2}. We show that gradient-descent training in Emb--MLP exhibits an implicit nuclear-norm bias, favoring low-rank parameters that share structure across tasks, and we prove that this bias is sufficient to transfer from one-hop training to OOD two-hop generalization even in a single layer.

At the same time, the one-layer solution we uncover for two-hop reasoning is clearly a form of shortcut pattern, rather than the kind of step-by-step implicit reasoning. Therefore, we further investigated a high-complexity setting with more relationships and empirically demonstrated that the performance degradation of standard GPT-2 is similar to that of the Emb--MLP model, suggesting that Emb--MLP is a good mechanism proxy. We additionally observe the same degradation across LLaMA-style, Qwen3-style, and Gemma-style models: all three are near-perfect at complexity $C=1$, but their OOD accuracy decreases steadily as the number of compositional branches grows. Moreover, we noticed the small initialization technique \citet{zhang2024initialization,zhang2025complexity,yao2025analysisreasoningbiaslanguage} or weight decay is useful to improve OOD generalization in high-complexity regimes. Our analysis links this improvement to tight representation alignment across layers, thereby confirming and extending the explanation in \citet{hoptwolate}.

To sum up, our contribution can be summarized as follows.

\begin{enumerate}[leftmargin=*]
\item 
We introduce the identity bridge, a minimal modification to the training data that adds a zero-hop task and reliably enables OOD two-hop composition in synthetic settings (Fig.~\ref{fig::acc_on_diff_models}).

\item Theoretically, We develop a uniform-attention theory (Theorems~\ref{thm:pos-ood} and~\ref{thm:neg-ood}) showing how identity bridge induces memorization that links $a$ to $c$, and why the standard setup without identity supervision fails.
\item 
Empirically, We show that Emb--MLP is a good proxy to standard transformer model, and that small initialization improves cross-layer alignment and OOD generalization (Fig.~\ref{fig:figure_alignment_along_epoch}). Controlled experiments with modern transformer models further show that the failure pattern persists across modern architectural choices (Fig.~\ref{fig::acc_on_diff_models}). We also  provide evidence that modern LLMs exhibit widespread shortcut-like memorization even without explicit two-hop supervision (Fig.~\ref{fig:layer_-1_novel_city_afterSFT}).

\end{enumerate}

\section{Related Work}
\label{sec:related}

\paragraph{Implicit reasoning failures on 2-hop data} 

\citet{allen-zhu2025physics_3_2} studies a phenomenon in which transformers face difficulty in
    manipulating already learned knowledge. \citet{press2022measuring} constructs a two-hop reasoning task on a knowledge graph, demonstrating that transformers can generalize to in-distribution data through long-term training, a phenomenon known as grokking. However, out-of-distribution data cannot be generalized. \citet{2025How} further investigated this phenomenon, demonstrating that in-distribution generalization stems from the presence of bridge entities in the two-hop task in the training set, thereby inducing alignment. However, these works did not propose methods to alleviate the generalization difficulties of OOD.

\paragraph{Compositionality gap in LLMs}
A growing literature documents a robust compositionality gap in LLMs on multi-hop tasks. \citet{press2022measuring, xu2024do} report a substantial accuracy drop from single-hop to two-hop queries, which does not vanish with increased model scale. \citet{2024Do} found limited evidence for implicit reasoning in large models by eliminating shortcuts, and \citet{2024DoReallyThinkStep} found a similar phenomenon in fine-tuning. \citet{DBLP:journals/corr/abs-2301-11293} also found that the model is more able to utilize popular knowledge rather than lesser-known knowledge during fine-tuning, which affects the model's reasoning performance. These works have demonstrated that it is possible for models to exploit shortcuts instead of genuine inference. To mechanistically explain these failures, \citet{hoptwolate} analyze intermediate transformer states via circuit methods and suggest that first-hop information may be formed too late to be utilized for composition. In contrast, \citet{liu2026layer} challenge hop-aligned circuit explanations by uncovering layer-order inversion, where answer information can become decodable earlier than the bridge entity.

\section{Preliminaries}
\subsection{two-hop Reasoning Task Setup}

We study two-hop reasoning as the composition of two mappings over disjoint entity sets based on \citep{2024Grokked,2025How}. Let
$\mathcal{E}_1$, $\mathcal{E}_2$, and $\mathcal{E}_3$ denote the subject, bridge, and answer entity sets, respectively, and let
$R_1$ and $R_2$ denote the first-hop and second-hop relation sets. We define two hop maps and associated two-hop compositional mapping as
\[
g_1: \mathcal{E}_1 \times R_1 \to \mathcal{E}_2,
\quad
g_2: \mathcal{E}_2 \times R_2 \to \mathcal{E}_3
\quad 
g(e_1,r_1,r_2) := g_2(g_1(e_1,r_1), r_2).
\]

Accordingly, the synthetic dataset contains the following three types of supervision:
(i) first-hop facts $(e_1,r_1) \mapsto g_1(e_1,r_1)$,
(ii) second-hop facts $(e_2,r_2) \mapsto g_2(e_2,r_2)$, and
(iii) identity-bridge facts $e_2 \mapsto e_2$ for bridge entities $e_2 \in \mathcal{E}_2$.
The OOD evaluation task consists of two-hop queries
\[
(e_1,r_1,r_2) \mapsto g(e_1,r_1,r_2),
\]
which are never provided as supervised training targets unless otherwise specified.

For intuition, one may think of $(a,r_1)\mapsto b$ as the first hop, $(b,r_2)\mapsto c$ as the second hop, and $(a,r_1,r_2)\mapsto c$ as the composed query. The identity bridge $b \mapsto b$ explicitly enforces alignment between the bridge token as an output of the first hop and as an input to the second hop.

\paragraph{Data complexity.} 
We define the complexity of the synthetic two-hop task as the number of first-hop relation types. Concretely, we fix $N \in \mathbb{N}$ as the number of entities, and set $|\mathcal{E}_1|=|\mathcal{E}_3|=N$, $|R_2|=1, |R_1|=C$, and partition the bridge set into $C$ disjoint slices. This definition captures the branching factor of the compositional task: a larger complexity means that each subject participates in more possible two-hop compositions.

Each first-hop relation $r_{1,j}$ selects the $j$-th bridge slice, while the shared second-hop relation $r_2$ maps bridge entities to answer entities. Therefore, for each subject entity, different choices of $r_{1,j}$ induce $C$ distinct two-hop compositions, so the dataset complexity is exactly $C$.

Given the symmetry between the two hops, we maintain an identical configuration for $r_2$ and focus our complexity analysis on $r_1$. Unless otherwise specified, training uses only one-hop and identity-bridge supervision, while all two-hop queries are reserved for OOD evaluation.

\subsection{Model Architecture}
Let $d_{\text{vob}}$, $d_m$, $d_k$ denote the vocabulary size, embedding dimension, and query-key projection dimension, respectively. The input vocabulary size and output are denoted by $\mathcal{V}_{\text{in}}$ and $\mathcal{V}_{\text{out}}$. The embedding matrix is $\mE \in \mathbb{R}^{|\mathcal{V}_{\text{in}}| \times d_m}$, with row $\mE_x$ corresponding to token $x \in \mathcal{V}_{\text{in}}$, and the projection matrix is $\mW_{\mathrm{proj}} \in \mathbb{R}^{d_m \times |\mathcal{V}_{\text{out}}|}$.

\paragraph{Embedding-MLP (Emb--MLP).} 
For tasks where attention serves only as an information-mixing mechanism, we adopt the Embedding-MLP model \citet{yao2025analysisreasoningbiaslanguage,huang2025generalization}, which can be viewed as a transformer layer with uniform attention. This formulation allows flexible handling of the input and output vocabularies. For a sequence $X=(x_1,\dots,x_s)$, we define:

\begin{defi}[Embedding--MLP (Emb--MLP)]
Given parameters $\vtheta=(\mE,\mW_{\mathrm{proj}})$, the model outputs logits is
\[
f_{\vtheta}(X) = \left(\sum_{t=1}^{s}\mE_{x_t}\right) \mW_{\mathrm{proj}} \in\mathbb{R}^{|\mathcal{V}_{\text{out}}|}.
\]
\end{defi}
\paragraph{Modern decoder variants.}
To test whether the observed behavior depends on the GPT-2 architecture, we additionally evaluate three decoder-only variants: LLaMA-style (Pre-RMSNorm, SwiGLU feed-forward layers, and RoPE), Qwen3-style (Pre-RMSNorm, SwiGLU, RoPE, and grouped-query attention), and Gemma-style (Pre-RMSNorm, GeGLU, RoPE, and multi-query attention). Each model uses 12 layers, 12 query heads, $d_{\mathrm{model}}=768$, and $d_{\mathrm{ff}}=1024$. We train each setting for 20 epochs using a learning rate of $10^{-4}$, batch size 500, weight decay 0.1, and cosine annealing with warmup, and report mean and standard deviation over three seeds.

For any learnable weight matrix $\mW \in \sR^{d_1 \times d_2}$, with $d_1$ and $d_2$ denoting the input and output dimensions, we initialize each entry from a Gaussian distribution $\mW_{i,j} \sim \mathcal{N}\!\left(0, \sigma^2 \right)$. The standard GPT-2 model take $\sigma = 0.02$. When we use a small initialization, we let $\sigma = d_1^{-\gamma}$, where $\gamma > 0.5$ since related work \citep{zhang2024initialization, yao2025analysisreasoningbiaslanguage} shows that networks initialized in this way often have stronger reasoning capabilities.

\section{Main Results}

We identify an overlooked input-output inconsistency and introduce explicit bridging entity supervision, identity bridge, to recover implicit two-hop reasoning in transformer. To explain this, we analyze Emb--MLP logits and provide a rigorous theoretical account of the underlying mechanism. Finally, we extend these insights to two-hop reasoning in LLMs.

\subsection{Motivation: the necessity of the identity bridge}

We begin by illustrating the core intuition with a minimal example. The key observation is that, in the absence of additional constraints, the model does not naturally align its outputs with its inputs.

To make this issue concrete, we utilize activation patching \citet{NEURIPS2022_6f1d43d5,lv2024interpretingkeymechanismsfactual,heimersheim2024useinterpretactivationpatching} to examine the efficacy of bridge token $b$. We consider the GPT-2 model at complexity $=2$. We feed the sequence $(a, r_1)$ into the model and extract the hidden states at the $r_1$ position across different layers. These extracted states are then injected into the position of $b$ within the sequence $(b, r_2)$, allowing the model to proceed with the subsequent inference. If the model recognizes that the extracted hidden state is equivalent to the input $b$, it should be able to further decode the corresponding answer $(b, r_2)\to c$. We report the activation patching accuracy in Fig.~\ref{fig::acc_on_diff_models}(left). The result shows that the model completely fails to decode $c$ on OOD two-hop reasoning, which indicates token $b$ doesn't bridge the gap between first hop and second hop.

We attribute this failure to a contextual disconnect between the input and output spaces. Under a simple two-hop task setup, the model is not explicitly required to establish an equivalence between the input token $b$ and the output token $b$, which is a trivial capability for well-trained LLMs. Consequently, a straightforward solution to address this issue is to augment the training data with $b\to b$ "zero-hop" sequences, which we term as identity bridge.

\subsection{Identity bridge Works: Emb–MLP as a proxy for GPT-2}

\begin{figure*}[!htb]
\centering
\includegraphics[width=\linewidth]{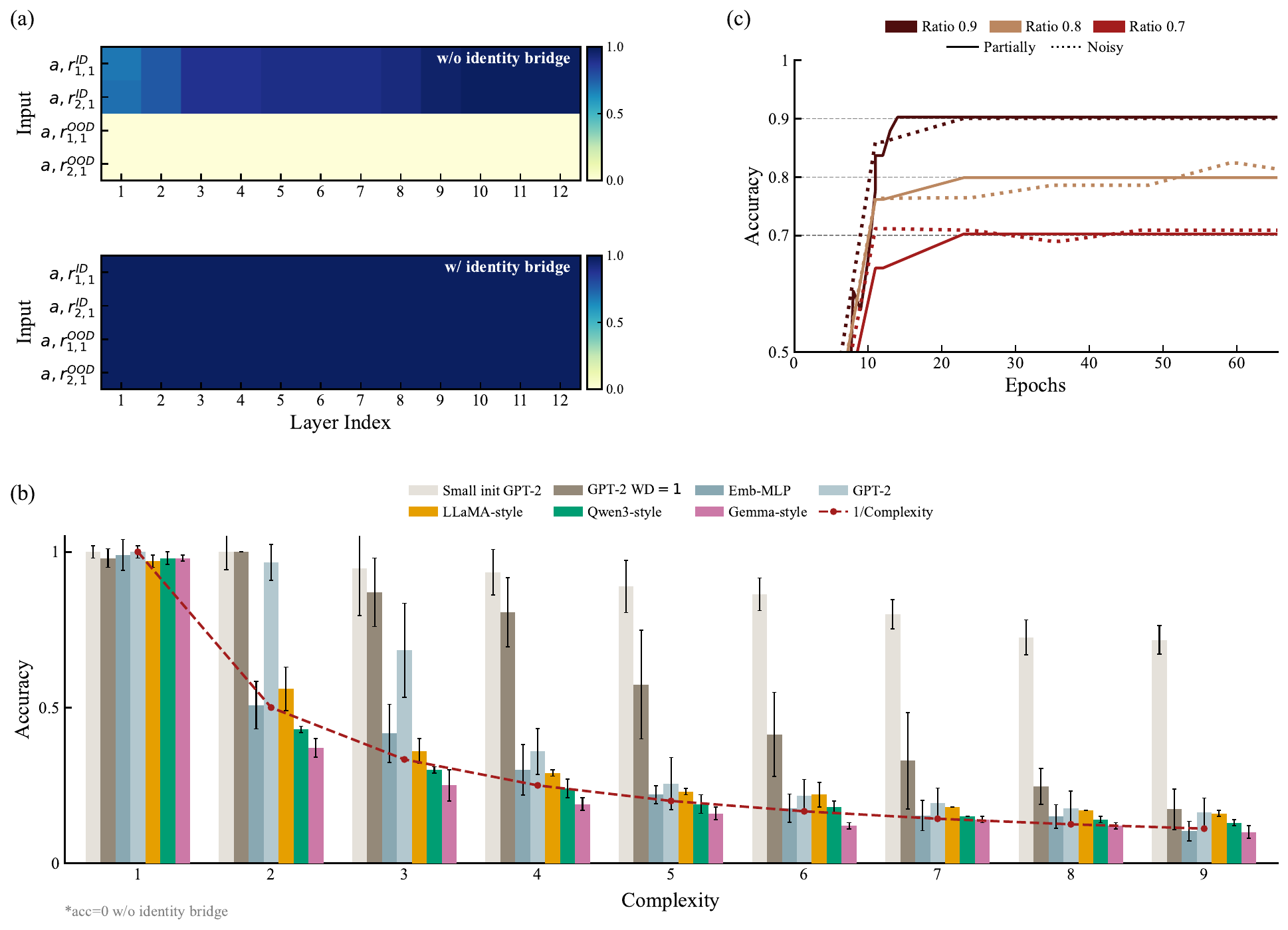}
\caption{{Left: Activation patching accuracy without (top) and with (bottom) the identity bridge. 
We report performance for both in-distribution (ID) relations ($r^{ID}$) and out-of-distribution (OOD) relations ($r^{OOD}$). Center: OOD two-hop accuracy from complexity $C=1$ to $C=9$ for Emb--MLP, standard and regularized GPT-2, and LLaMA-style, Qwen3-style, and Gemma-style models. For the three modern decoder variants, error bars denote standard deviation over three seeds.
Right: The effect of retaining a certain proportion of identity bridge, while masking or adding label noise to the remaining. The model's accuracy closely aligns with the proportion of retained data.}}
\label{fig::acc_on_diff_models}
\end{figure*}

Fig.~\ref{fig::acc_on_diff_models} reports test accuracy across models and complexity levels $C$. The identity bridge proves consistently effective, yielding non-zero OOD two-hop generalization for all models, irrespective of architectural nonlinearities or initialization schemes. Identity supervision suffices for the simplified Emb--MLP to achieve high performance, aligning with the implicit-regularization framework developed in Secs.~\ref{subsec::Cross-task memory} and~\ref{subsec::Uniform-attention theory}.

The architecture comparison in the center panel strengthens this conclusion. We set the LLaMA-style, Qwen3-style, and Gemma-style models at the same size. At $C=1$, models achieve nearly 1 OOD accuracy, respectively. As complexity increases to $C=9$, their accuracies decrease. The shared trend across different architectures indicates that the phenomenon is robust to modern decoder design choices.

Notably, as complexity $C$ increases, the accuracy of standard GPT-2 and Emb--MLP declines in tandem. This suggests that Emb--MLP effectively captures the mechanisms of the transformer, but also reveals that the implicit bias of gradient descent alone is insufficient for robust OOD composition at higher complexities. In contrast, the small-initialized GPT-2 degrades more gracefully, indicating a stronger regularization effect: the initialization-induced bias further constrains the latent geometry and better preserves the bridge--object coupling required for composition, a phenomenon we further discuss in Sec. \ref{sec:high-complexity}.

\subsection{Mechanism: how Emb–MLP solves two-hop reasoning}
\label{subsec::Cross-task memory}

We now examine how identity bridge enables composition of one-hop tasks. Let the input and output vocabularies be
\(\mathcal{V}_{\text{in}}=\mathcal{E}_1\cup\mathcal{E}_2\cup\mathcal{R}\) and
\(\mathcal{V}_{\text{out}}=\mathcal{E}_2\cup\mathcal{E}_3\), respectively. To make the mechanism explicit, we use the Emb--MLP as the proxy model and analyze the row-wise logit templates encoded by
\[
\mW \;=\; \mE\,\mW_{\mathrm{proj}} \ \in \ \mathbb{R}^{|\mathcal{V}_{\text{in}}|\times |\mathcal{V}_{\text{out}}|},
\]
where the $i$-th row of $\mW$ is the logit vector produced by input token $i$. $\mW_{ij}$ is the logit assigned by token $i$ to output token $j$. Furthermore, since Emb-MLP is a uniform attention proxy, to see the model's output logit for a sequence $X$, we only need to sum the corresponding rows of $\boldsymbol{W}$. This facilitates our observation of the composite mechanism in two-hop tasks. We also drew the output logit of standard GPT-2 model to show a similar structure.

\begin{figure*}[htbp]
\centering
\includegraphics[width=\linewidth]{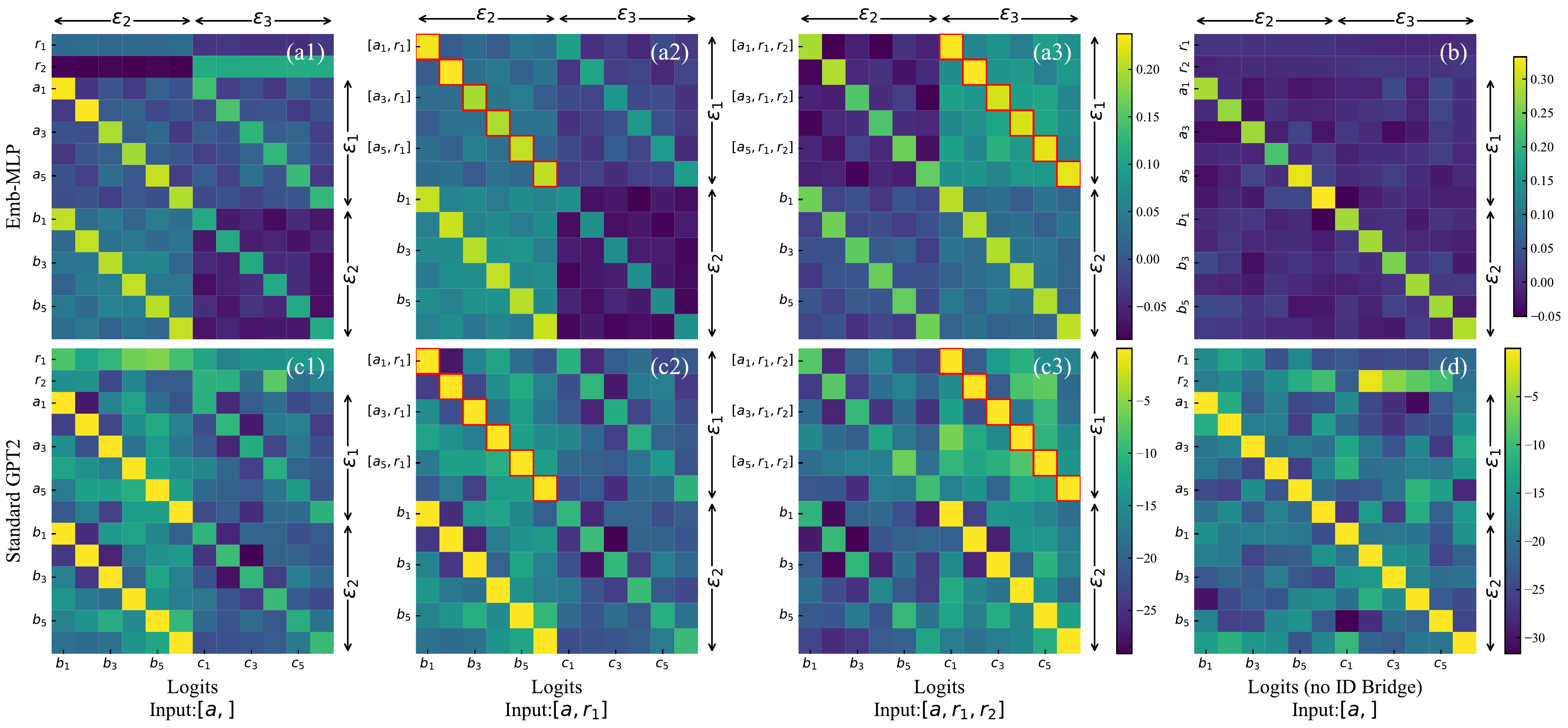}
\caption{\textbf{Row-wise logit templates in Emb--MLP and comparison with Standard GPT-2 at complexity $=1$ regime.}
Panels ({a1-a3}) and ({c1-c3}) are trained with identity bridge; panel (b, d) omits it.
\textbf{(a1)} Base logit matrix: the top two rows correspond to relation tokens $r_1$ and $r_2$; the remaining rows correspond to entity tokens in $\mathcal{E}_1$ and $\mathcal{E}_2$.
\textbf{(a2)} One-hop input $(a_i,r_1)$, visualized as the row-wise sum of $a_i$ and $r_1$.
\textbf{(a3)} Two-hop query $(a_i,r_1,r_2)$, visualized as the row-wise sum of $a_i$, $r_1$, and $r_2$.
Red boxes indicate the current argmax output token.
\textbf{(b, d)} Same visualization as in (a1) and (c1) but trained without identity supervision. }
\label{fig::logits}
\end{figure*}

Fig.~\ref{fig::logits}(a1) shows that relations act primarily as set selectors: $r_1$ boosts logits toward $\mathcal{E}_2$ while $r_2$ does the converse. Hence, the substantive computation of the two hops is carried by the entity rows of $\mW$. The key question is whether the subject rows for $a_i$ encode a discriminative bias toward the correct tail $c_i$.

Fig.~\ref{fig::logits}(a1) further indicates that each bridge token $b_i$ is both self-peaked and object-aligned, exhibiting high logit on $b_i$ and on its paired $c_i$. Training on $(a_i,r_1)$ would only concentrate the subject's logits on the corresponding bridging slices. However, under the implicit nuclear norm regularization introduced by gradient-based training, as discussed in Sec.~\ref{subsec::Uniform-attention theory}, the lowest-rank way to satisfy all constraints shares this structure across blocks, effectively transferring the bridge’s object-aligned peak to the subject rows. Consequently, the rows associated with $a_i$ inherit a tail-directed bias via their linkage to $b_i$, supplying the cross-task memory needed for composition. Figs.~\ref{fig::logits}(a2) and \ref{fig::logits}(a3) show that the model then completes both the one-hop task and the two-hop generalization by combining subject entities with relations. 

In contrast, without identity, non-label logits within a block tend to equalize, yielding a nearly diagonal-dominant pattern that conveys little information about the correct object $c_i$ for a given subject $a_i$, thereby leading to failure of two-hop reasoning.

To further demonstrate that Emb--MLP is a good approximation of the standard transformer model, we visualized the logit of the standard GPT-2 model. The logit structures for single tokens, one-hop data, and two-hop data all show consistency with Emb--MLP, indicating that although there are differences between the Emb--MLP model and the transformer model, the Emb--MLP model reproduces mechanism of the transformer  well in this task.

\subsection{Theory: uniform-attention analysis of two-hop generalization}
\label{subsec::Uniform-attention theory}
We analyze how identity bridge enables two-hop generalization in the Emb--MLP model on the dataset with complexity one, where attention acts only as uniform mixing. Previous experimental evidence has shown that this model contains similar mechanisms to the standard GPT-2 model. 

Before presenting the main results, we introduce some necessary concepts and related results. For a labeled example \((X,y)\) with \(y\in\mathcal{V}_{\text{out}}\), define the pairwise logit gap and multiclass margin by
\begin{equation}
s_{(X,y),y'} = f_{\vtheta}(X)_y - f_{\vtheta}(X)_{y'} \quad q(X,y)=\min_{y'\in\mathcal{V}_{\text{out}}\setminus\{y\}} s_{(X,y),y'} .    
\end{equation}
Because Emb--MLP model is positively homogeneous in its parameters, and under standard separability with cross-entropy training, the normalized direction \(\vtheta/\|\vtheta\|_2\) converges to a KKT point of the margin-maximization program established by \citet{lyu2019gradient}:
\begin{equation}
\label{equ::max_margin}
\min_{\vtheta} \tfrac12\|\vtheta\|_2^2,   \quad
\text{s.t.} \
s_{(X,y),y'}\ge 1, \quad \forall (X,y)\in\mathcal{D}_{\text{train}},\quad \forall y'\in\mathcal{V}_{\text{out}}\setminus\{y\}.    
\end{equation}
In the Emb--MLP model, {the logit matrix $\mW$ can be viewed as a reparameterization of $\mE$ and $\mW_{\mathrm{proj}}$}, which leads to the following {equivalent} convex reformulation in \(\mW\) introduced by \citet{huang2025generalization}:
\begin{equation}
\label{equ::min_nuclear_norm}
\min_{\mW} \ \tfrac12\|\mW\|_*^2 ,   \quad
\text{s.t.} \ 
s_{(X,y),y'}\ge 1,
\ \quad \forall (X,y)\in\mathcal{D}_{\text{train}},\quad \forall y'\in\mathcal{V}_{\text{out}}\setminus\{y\}. 
\end{equation}
Problem~(\ref{equ::min_nuclear_norm}) is equivalent to problem~(\ref{equ::max_margin}), as established by Lemma~1 in \citet{huang2025generalization}, whose proof follows an idea similar to that of \citet{recht2010guaranteed}. We briefly sketch the main argument here and refer readers to \citet{huang2025generalization} for full details. For Emb--MLP, the objective of problem~(\ref{equ::max_margin}) takes the form $$\min_{\mE,\mW_{\mathrm{proj}}} \frac{1}{2} \left( \|\mE\|_{\mathrm{F}}^2 + \|\mW_{\mathrm{proj}}\|_{\mathrm{F}}^2\right).$$ The equivalence follows from the well-known variational characterization linking the nuclear norm and the Frobenius norm, $$\|\mW\|_*^2 = \min_{\{ \mE \mW_{\mathrm{proj}} = \mW\}} \frac{1}{2} \left( \|\mE\|_{\mathrm{F}}^2 + \|\mW_{\mathrm{proj}}\|_{\mathrm{F}}^2\right)$$
which completes the argument.

This equivalent convex reformulation greatly facilitates the analysis. The problem is convex, as the objective function is convex and the constraints are affine. We therefore work with problem~(\ref{equ::min_nuclear_norm}) to derive margin-based consequences for two-hop generalization.

We now state the formal consequences for two-hop generalization.
\begin{thm}[Positive OOD margin with identity supervision]
\label{thm:pos-ood}
{Given Assumption \ref{assump::uniqueness} and large $N$}. Assume training includes zero-hop identity supervision over \(\mathcal{E}_2\), and let \(\mW^\star\) (equivalently \(\vtheta^\star\)) solve problem (\ref{equ::min_nuclear_norm}).
Then for every OOD query \(X=(a_i,r_1,r_2)\) with label \(y=c_i\), the multiclass margin is positive:
\begin{equation*}
q(X,y)\;>\;0.
\end{equation*}
Hence, the composed mapping \(g_2\!\circ g_1\) is recovered in the OOD two-hop task.
\end{thm}

\noindent\textbf{Proof sketch.}
We first use the permutation symmetry of the dataset to get a highly structured optimal solution and obtain a closed-form objective function in this structure. {By further utilizing symmetry, we proved a set of symmetry conditions (Proposition \ref{prop::symmetry_prop}) that an optimal solution must satisfy, further simplifying the problem. Next, using the construction method and proof by contradiction, we strengthened the conditions (Proposition \ref{prop::positive_c_1} and Corollary \ref{coro::tight_g_2}) that the optimal solution must satisfy. } 
These links shrink the OOD margin check to a one-dimensional inequality that feasibility makes strictly positive, yielding a positive margin on every two-hop query and the success of OOD generalization. See appendix \ref{subsec::proof_for_thm1} for details of the proof.

\begin{thm}[Failure without identity supervision]
\label{thm:neg-ood}
{Given Assumption \ref{assump::uniqueness}.} If identity supervision is omitted, any problem of (\ref{equ::min_nuclear_norm}) satisfies, for each OOD query \(X=(a_i,r_1,r_2)\) with label \(y=c_i\), the multiclass margin is negative:
\begin{equation*}
q(X,y)\;<\;0.
\end{equation*}
Thus, the composed mapping fails on the OOD two-hop task.
\end{thm}

\noindent\textbf{Proof sketch.}
Without identity supervision, the training constraints are perfectly symmetric inside each block. 
Because the objective is convex and permutation-invariant, we can average any optimal solution over all within-block permutations and obtain an equally optimal, fully symmetric one. 
A short KKT check then shows the non-label logits equalize within blocks, so subjects carry no preference toward their true objects. 
When we test a held-out two-hop composition, the signal from the subject does not point to the correct tail and the margin becomes negative. 
Hence the margins are negative for all OOD two-hop queries and composition fails. See Appendix \ref{subsec::proof_for_thm2} for details of the proof.

\subsection{Real Task Analysis}
\begin{figure*}[!htb]
    \centering
    \includegraphics[width=0.9\linewidth]{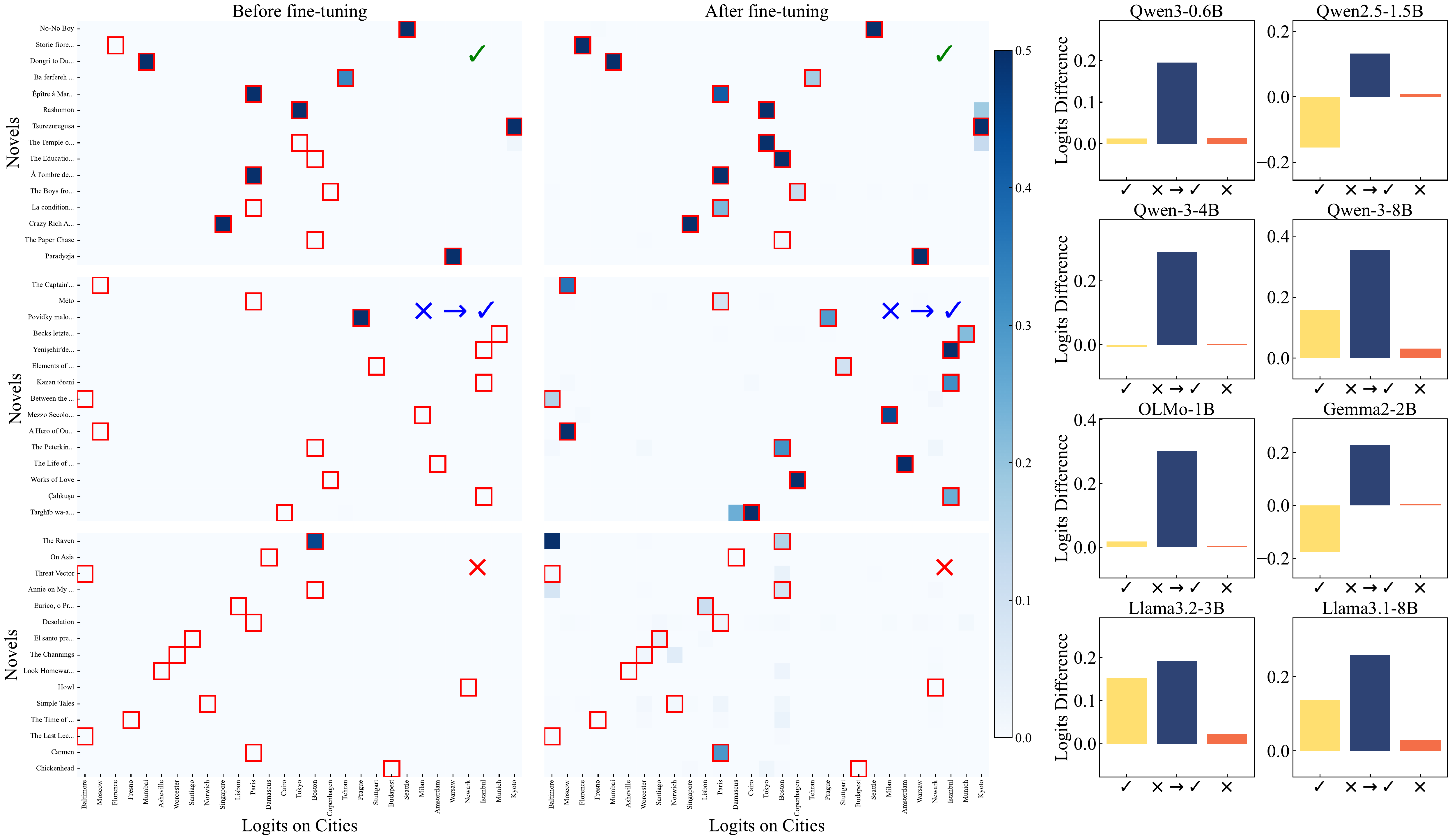}
    \caption{The probability of the model outputting the alternative city token when using the prompt corresponding to $(e_1, r_2)$ before and after fine-tuning on datasets. The red box indicates the answer to the corresponding two-hop reasoning data. The bar chart shows the change in the probability of the corresponding two-hop answer when using the prompt $(e_1, r_2)$ for different datasets. }
    \label{fig:layer_-1_novel_city_afterSFT}
\end{figure*}

In this section, we consider the mechanism of two-hop reasoning in real large models. As pointed out in \citet{press2022measuring}, despite the existence of a combinatorial gap, real large models still have two-hop reasoning capabilities. We finetune pretrained models on TWOHOPFACT dataset \citet{2024Do} to verifying our theoretical results. We filtered the data based on the bridge entity to ensure the OOD characteristics of the test data.
\paragraph{LLMs learn identity bridge in pretraining}
We perform a repeat task to monitor for identity bridges that may appear in the language model. Conceptually, the core requirement for the identity bridge is the model's ability to map the input and output of the bridge token $b$. Thus, we use the prompt ``Repeat the following token $\{ b\}$ directly'' with 1-shot example to analyze whether the model possesses an identity bridge-like capability during pretraining. We utilize the Olmo2-0425 \citet{olmo20242} version checkpoints for test and results are shown in Fig.~\ref{fig:IDbridge_test}.

\paragraph{LLMs learn $a\to c$ after fine-tuning}
For large models, since the parameters have been fully pre-trained, the alignment phenomenon is unlikely to be observed from the hidden state. The model should basically rely on the implicit regularization induced by the gradient descent algorithm to complete the task. In order to observe this mechanism, we extract the following three datasets based on the two-hop results for analysis: $\mathcal{D}_{\text{correct}}$ consists of data with correct two-hop reasoning whether fine-tuning or not, $\mathcal{D}_{\text{partial}}$ contains the data that are wrong in the first two hops of fine-tuning but correct after fine-tuning, and $\mathcal{D}_{\text{incorrect}}$ contains the data of fine-tuning the errors of the two-hop task.

As Fig. \ref{fig:layer_-1_novel_city_afterSFT} shown, after training on the single-hop task, even without seeing the corresponding two-hop data, for the correct data after training, when we use prompts such as $(e_1, r_2)$, such as ``the novel was born in the city of" but omit the ``author" relation, the model still establishes a strong correlation between the subject and the object, suggesting that the model actually implicitly establishes the identity bridge during the pre-training process which utilizes similar implicit regularization. 

We calculated the probability change of labels corresponding to two-hop data using this type of prompt. The results from different models consistently show that completing two-hop reasoning depends on improving the probability of the subject to the object. For the data that still got it wrong after fine-tuning, we found that the probability of the corresponding object was slightly improved. The experiments on other tasks such as MuSiQue dataset \citet{trivedi2021musique} could be referred at Sec.~\ref{sec:other_real_world_tasks}.

\begin{figure*}[!htb]
  \centering
  \begin{subfigure}[t]{0.68\textwidth}
    \centering
    \includegraphics[width=\linewidth]{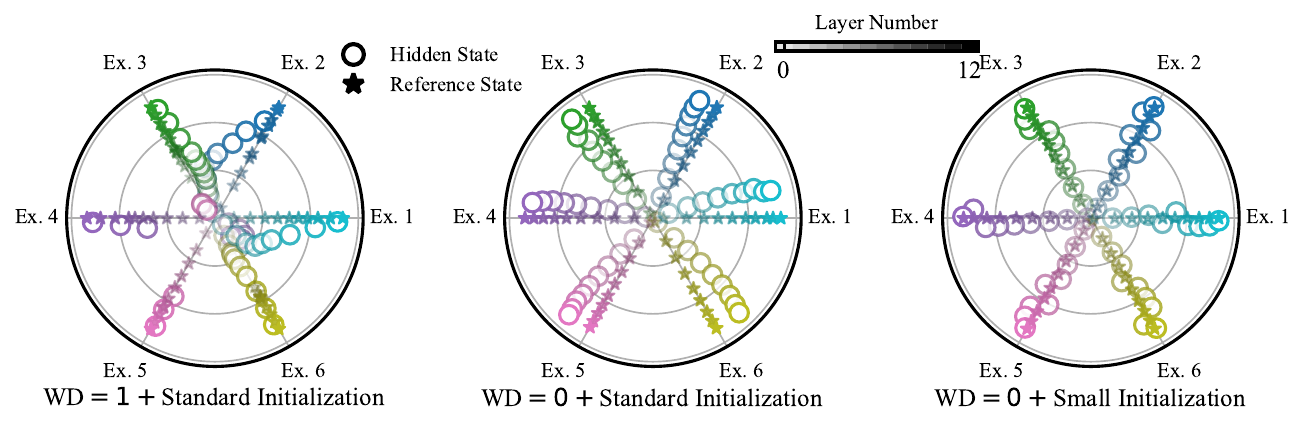}
    \caption{}
\label{fig:FIGURE_3_wd+init1.pdf}
  \end{subfigure}\hfill
  \begin{subfigure}[t]{0.32\textwidth}
    \centering
    \includegraphics[width=\linewidth]{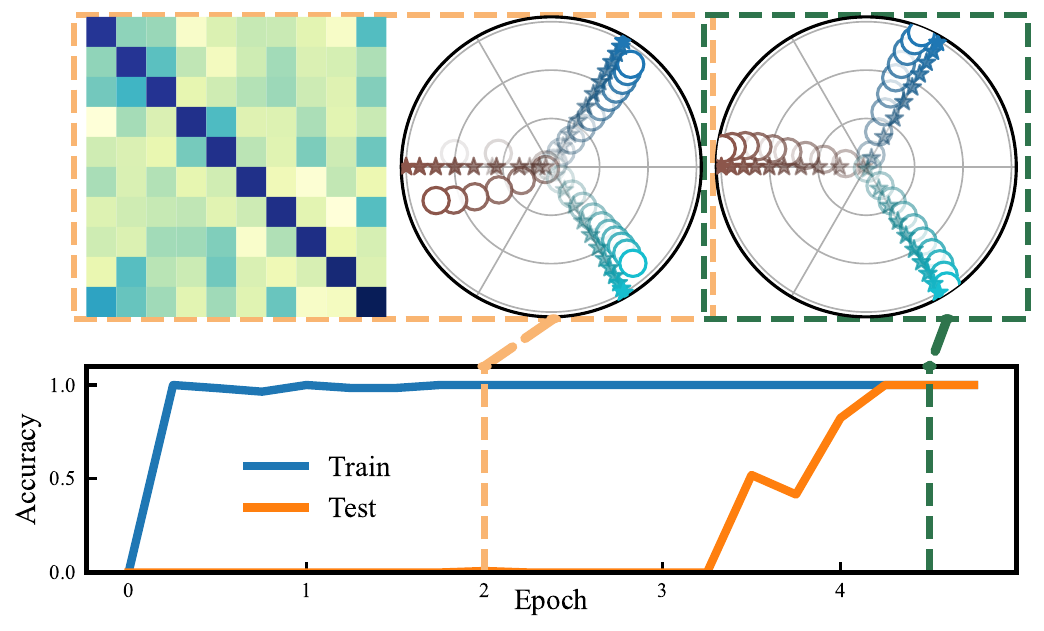}
    \caption{}
    \label{fig:figure_alignment_along_epoch}
  \end{subfigure}
  \caption{(a) T-SNE visualization of hidden states with respect to one-hop data and corresponding bridge entity. Six samples are selected for each setting, where the star represents the hidden state of the bridge entity and the circle represents the hidden state of the first hop data $(e_1, r_i)$ at position $r_i$. Colors from light to dark represent hidden states from shallow to deep. (b) Alignment analysis along the training steps for small initialized GPT-2 with complexity$=2$. The two images correspond to when the memorization on 1-hop ends but OOD generalization ability is missing, and the moment when OOD generalization is completed. Figure (a1) and (b1) shows the relationship between the input $a_i$ and the output logits of $c_i$, similarly to Fig.~\ref{fig::logits}.}
\label{fig:FIGURE_3_wd+init.pdf}
\end{figure*}

\section{Another direction: strengthening regularization and alignment}
\label{sec:high-complexity}

As dataset complexity increases, the implicit regularization is no longer sufficient to solve two-hop reasoning task since relying solely on identity bridge mechanism. We therefore strengthen regularization either by small initialization or by weight decay. Both interventions substantially recover OOD performance. 

To further investigate whether the model truly learns implicit reasoning rather than merely adopting shortcut patterns, we analyze the alignment of hidden states at the last token position across different layers using t-SNE dimensionality reduction for inputs $(e_1, r_1)$ and $e_2$. As illustrated in Figure 4(a), taking the hidden state of $e_2$ as a baseline, we observe that the hidden states of $(e_1, r_1)$ increasingly approach those of $e_2$ with large weight decay or small initialization. This trend suggests that the network progressively learns to associate the $(e_1, r_1)$ pair with the bridge entity $e_2$. In contrast, standard GPT-2 models do not exhibit such alignment, indicating that without the aid of proper regularization, the model fails to recognize the intrinsic connection between entities.

Furthermore, as Fig. \ref{fig:figure_alignment_along_epoch}(b) indicates, the emergence of identity bridge effects precedes the growth of generalization ability. With the alignment of hidden states improves during training, the test accuracy rises in step with this alignment trend. In other words, when the representations for one-hop data and the target bridge collapse into the same subspace, the second hop can reliably latch onto the correct features and composition succeeds.

\section{Discussion}

\paragraph{Conclusion} 
In this work, we revisit the two-hop curse phenomenon. We show that a minimal modification to the training data, the identity bridge, is sufficient to unlock OOD two-hop composition. Through a rigorous analysis of a simplified Emb--MLP model, we prove that identity bridge provably improves OOD two-hop generalization, while its absence inevitably leads to failure.

Empirically, we find that Emb--MLP closely approximates the behavior of standard GPT-2, indicating that standard transformers also tend to solve two-hop tasks via shortcut  rather than genuine step-by-step reasoning. Experiments with modern transformer models reproduce the same accuracy--complexity trade-off, broadening the evidence beyond GPT-2. Experiments on real language tasks suggest that pretraining already endows modern LLMs with implicit identity-bridge-like capabilities and help establish a subject-to-answer relationship. To move toward more faithful implicit reasoning, we investigate additional regularization methods.

\paragraph{Limitations}
Our theoretical analysis is conducted within a simplified Emb--MLP model with a uniform attention mechanism and does not explicitly capture the full attention dynamics of transformers. Furthermore, it focuses on two-hop reasoning in synthetic environments. Nevertheless, experiments show that this model tracks GPT-2 well across various complexity levels and reproduces its success and failure patterns with and without identity supervision. The additional architecture comparison reduces, but does not eliminate, this limitation: the same trend holds for three modern decoder variants, while substantially larger pretrained models and non-decoder architectures remain outside the controlled analysis.

\section*{Acknowledgements}
This work is also sponsored by the National Natural Science Foundation of China Grant No. 12371511, 12422119, 2025 Key Technology R\&D Program ``New Generation Information Technology'' Project 25511103100 of Shanghai Municipal Science and Technology Commission. 

\bibliography{colm2026_conference}
\bibliographystyle{colm2026_conference}
\newpage
\appendix
\section{Notations and Definitions}\label{sec::2hop-setup}
We adopt the following notation conventions.
All tokens in our synthetic tasks are encoded as positive integers.
Calligraphic capital letters (e.g., $\mathcal{E}, \mathcal{R}$) denote sets of tokens,
roman capital letters (e.g., $N, C$) denote scalar parameters such as dataset size and complexity,
and lower-case letters denote individual tokens.
The entity set $\mathcal{E}$ is partitioned into subject, bridge, and object subsets
$\mathcal{E}_1, \mathcal{E}_2, \mathcal{E}_3$, with typical elements
$e_{1,\cdot} \in \mathcal{E}_1$, $e_{2,\cdot} \in \mathcal{E}_2$, and $e_{3,\cdot} \in \mathcal{E}_3$.
The relation set $\mathcal{R}$ is split into first- and second-hop families
$\mathcal{R}_1$ and $\mathcal{R}_2$, with elements
$r_{1,\cdot} \in \mathcal{R}_1$ and $r_{2,\cdot} \in \mathcal{R}_2$.
For any set $S$, we write $|S|$ or $\#S$ for its cardinality.

\subsection{Two-Hop Reasoning Task Setup}
To study the mechanism behind compositionality gap, we introduce the synthetic dataset in which all tokens are positive integers partitioned into a disjoint set of entities \((\mathcal{E})\) and relations \((\mathcal{R})\).
Entities are split into subject, bridge, and object subsets:
\[
\mathcal{E} \;=\; \mathcal{E}_1 \cup \mathcal{E}_2 \cup \mathcal{E}_3,
\quad\text{with}\quad
\mathcal{E}_i \cap \mathcal{E}_j = \emptyset \ (i\neq j).
\]
Relations are divided into two disjoint families for the two hops:
\[
\mathcal{R} \;=\; \mathcal{R}_1 \cup \mathcal{R}_2,
\quad\text{with}\quad
\mathcal{R}_1 \cap \mathcal{R}_2 = \emptyset.
\]

\paragraph{One-hop tasks.}
We instantiate two one-hop tasks. For the first hop, we first partition the bridge entities $\mathcal{E}_2$ according to {complexity determined by the number of} relations in \(\mathcal{R}_1\):
\[
\mathcal{E}_2 = \bigcup_{j_1 =1}^{|\mathcal{R}_1|} \mathcal{E}_{2,j_1}.
\]
Then define a deterministic (or sampling) map
\[
g_1:\; \mathcal{E}_1 \times \mathcal{R}_1 \to \mathcal{E}_2
\quad\text{such that}\quad
g_1(e_1,r_{1,j_1}) \in \mathcal{E}_{2,j_1}.
\]
The first-hop triple set is
\[
\mathcal{T}_1 \;=\; \{\, (e_1, r_{1,j_1}, e_2)\ :\ e_1 \in \mathcal{E}_1,\ r_{1,j_1} \in \mathcal{R}_1,\ e_2 = g_1(e_1,r_{1,j_1}) \,\}.
\]
This partitioning of \(\mathcal{E}_2\) ensures that each \(r_{1,j_1} \in \mathcal{R}_1\) only co-occurs with bridge entities from its dedicated slice \(\mathcal{E}_{2,j_1}\), reducing spurious shortcuts across relations. Second-hop construction follows analogous principles.

\paragraph{Two-hop composition.}
A two-hop instance composes the two one-hop maps. Given \((e_1, r_{1,j_1}, e_2) \in \mathcal{T}_1\) and \((e_2, r_{2,j_2}, e_3) \in \mathcal{T}_2\), the composed query is \((e_1, r_{1,j_1}, r_{2,j_2})\) with answer
\[
(e_1, r_{1,j_1}, r_{2,j_2}) \;=\; g_2\bigl(g_1(e_1,r_{1,j_1}),\, r_{2,j_2}\bigr) \;=\; e_3.
\]

\paragraph{Identity Bridge.}
Unlike prior work, we also include a zero-hop task over bridge entities called identity bridge to establish the connection between two one-hop tasks and shape the model’s latent space.
For each bridge entity \(e_2 \in \mathcal{E}_2\) we add a training pair of the form
\[
( e_2 ) \;\to\; e_2,
\]
encouraging the model to implement an identity transformation \(f(e_2) = e_2\) on the bridge entities. This task is not introduced for its standalone difficulty, but to regularize representations so that the composed mapping \(g_2 \circ g_1\) can be more reliably recovered during two-hop generalization.

\subsection{Dataset Setup}\label{sec::data-definition}
\textbf{OOD two-hop reasoning.}
After determine the $g_1$ and $g_2$ maps, we can construct training dataset $\mathcal{D}_{\text{train}}$ and test dataset $\mathcal{D}_{\text{test}}$ where $\mathcal{D}_{\text{train}}$ contains all one-hop data and partial two-hop data and $\mathcal{D}_{\text{test}}$ contains only the two-hop data. To investigate OOD composition ability, a two-hop $(e_1, r_1, r_2)$ data is called out-of-distribution if the corresponding bridge entity $e_2$ has never appeared in the two-hop data of the training set. In this work, unless otherwise specified, the training set is restricted to contain only single-hop data, ensuring that all two-hop data are out-of-distribution.

\textbf{Dataset Complexity}
By adjusting the configurations of $g_1$, $g_2$, and the number of relations, we can control the complexity of the dataset. The dataset complexity is defined precisely as the number of object entities associated with each subject as follows:
\begin{equation*}
    \text{Complexity} = \max_{e_1 \in \mathcal{E}_1} \# \{ e_3 = (e_1, r_{1,j_1}, r_{2,j_2}) \ | \ r_{1,j_1} \in \mathcal{R}_1, r_{2,j_2} \in \mathcal{R}_2 \}.
\end{equation*}
To construct different complexity data, we instantiate families of datasets with controlled complexity. Fix $N\in\mathbb{N}$ and set $|\mathcal{E}_1|=|\mathcal{E}_3|=N$. Let $C\in\mathbb{N}$ denote the complexity parameter and take
\[
|\mathcal{E}_2|=CN,\qquad |\mathcal{R}_1|=C,\qquad |\mathcal{R}_2|=1 .
\]
Partition the bridge set evenly as
\[
\mathcal{E}_2 \;=\; \bigcup_{j_1 =1}^{C}\,\mathcal{E}_{2,j_1}, 
\qquad
\mathcal{E}_{2,j_1}\;=\;\{\,e_{2,k}\,:\,(j_1-1)N+1 \le k \le j_1 N\,\}.
\]
Write $\mathcal{R}_1=\{r_{1,1},\ldots,r_{1,C}\}$ and $\mathcal{R}_2=\{r_2\}$. The hop maps are defined by
\[
g_1(e_{1,i_1},\,r_{1,j_1}) \;=\; e_{2,(j_1-1)N+i_1}, 
\qquad 
g_2(e_{2,k},\,r_2) = e_{3, \left\lfloor \frac{k}{N} \right\rfloor +k\bmod N} ,
\]
so that each $r_{1,j}$ selects the $j$-th bridge slice and $r_2$ collapses the bridge index modulo $N$ onto $\mathcal{E}_3$. When $C=1$, we construct the structured setting used for analysis:
\[
\mathcal{E}_1=\{a_1,\ldots,a_N\},\quad
\mathcal{E}_2=\{b_1,\ldots,b_N\},\quad
\mathcal{E}_3=\{c_1,\ldots,c_N\},\quad
\mathcal{R}_1=\{r_1\},\ \mathcal{R}_2=\{r_2\},
\]
with one-to-one correspondences $g_1(a_{i_1},r_1)=b_{i_1}$ and $g_2(b_{i_2},r_2)=c_{i_2}$, hence the composed answer for $(a_{i_1},r_1,r_2)$ is $c_{i_1}$.

\section{Theoretical Details}
\label{sec::theoretical_details}
This appendix completes the proofs of Theorems~\ref{thm:pos-ood} and \ref{thm:neg-ood}. 
We first collect several standard tools and assumptions from the literature that will be used throughout the arguments, and then give the proofs via a detailed analysis of the nuclear-norm program—combining a constructive step with a contradiction argument.

\subsection{Preliminaries for KKT Conditions}

First we recall the notion of subdifferentials for convex functions and then state the KKT conditions in subdifferential form.

Let $\varphi : \mathbb{R}^d \to (-\infty,+\infty]$ be a proper convex function. The subdifferential of $\varphi$ at $\vx \in \mathbb{R}^d$ is defined by
\[
    \partial \varphi(\vx)
    :=
    \left\{
        \vs \in \mathbb{R}^d :
        \varphi(\vy) \ge \varphi(\vx) + \langle \vs, \vy - \vx \rangle
        \quad \forall\, \vy \in \mathbb{R}^d
    \right\}.
\]

Consider the following optimization problem $(\text{P})$ for $\vx \in \mathbb{R}^d$:
\begin{equation*}
    \begin{aligned}
        \min \quad & f(\vx) \\
        \textrm{s.t.} \quad & g_n (\vx) \ge 0   \quad \forall n \in [N],\\
                            & h_m (\vx) = 0 \quad \forall m \in [M],
    \end{aligned}
\end{equation*}
where $f, g_n, h_m$ are convex and finite-valued on their domains. We say that $\vx \in \mathbb{R}^d$ is a feasible point of $(\text{P})$ if $\vx$ satisfies $g_n (\vx) \ge 0$ for all $n \in [N]$ and $h_m (\vx) = 0$ for all $m \in [M]$.

\begin{defi}[KKT point]
    A feasible point $\vx$ of $(\text{P})$ is called a KKT point if there exist multipliers
    $\lambda_1, \dots, \lambda_N \ge 0$ and $\mu_1, \dots, \mu_M \in \mathbb{R}$ such that
    \begin{enumerate}[leftmargin = *]
        \item (Subdifferential stationarity) There exist subgradients
        $\vs_f \in \partial f(\vx)$, $\vs_{g_n} \in \partial g_n(\vx)$ for all $n \in [N]$,
        and $\vs_{h_m} \in \partial h_m(\vx)$ for all $m \in [M]$ such that
        \[
            \vs_f
            \;-\;
            \sum_{n=1}^N \lambda_n \vs_{g_n}
            \;-\;
            \sum_{m=1}^M \mu_m \vs_{h_m}
            \;=\; 0.
        \]
        Equivalently,
        \[
            0 \in
            \partial f(\vx)
            - \sum_{n=1}^N \lambda_n \partial g_n(\vx)
            - \sum_{m=1}^M \mu_m \partial h_m(\vx).
        \]
        \item (Complementary slackness) For all $n \in [N]$,
        \[
            \lambda_n g_n(\vx) = 0.
        \]
    \end{enumerate}
\end{defi}

In general, global minimizers of $(\text{P})$ do not necessarily satisfy the KKT conditions. However, under suitable regularity assumptions (such as a Slater-type constraint qualification), the KKT conditions become necessary for optimality. To ensure the validity and clarity of the subsequent analysis, we impose the following standard assumptions.

Next we recall a standard constraint qualification for convex optimization, known as Slater's condition, which guarantees that KKT conditions are satisfied at optimal solutions.

\begin{defi}[Slater's condition]
    Consider problem $(\text{P})$ above and assume that
    $f$ and $g_n$ are convex functions on $\mathbb{R}^d$ for all $n \in [N]$, and
    each $h_m$ is an affine function on $\mathbb{R}^d$ for all $m \in [M]$.
    We say that Slater's condition holds for $(\text{P})$ if there exists a point
    $\bar{\vx} \in \mathbb{R}^d$ such that
    \[
        g_n(\bar{\vx}) > 0 \quad \forall n \in [N],
        \qquad
        h_m(\bar{\vx}) = 0 \quad \forall m \in [M].
    \]
    Such a point $\bar{\vx}$ is called a strictly feasible point.
\end{defi}

Under Slater's condition, the KKT conditions become necessary (and, in the convex setting, also sufficient) for optimality.

\begin{prop}[KKT conditions under Slater's condition]
    Suppose that $(\text{P})$ is a convex optimization problem in the sense that
    $f$ and $g_n$ are convex and $h_m$ are affine, and that Slater's condition holds.
    Let $\vx^*$ be a global minimizer of $(\text{P})$. Then $\vx^*$ is a KKT point; that is,
    there exist multipliers $\lambda_1^*,\dots,\lambda_N^* \ge 0$ and
    $\mu_1^*,\dots,\mu_M^* \in \mathbb{R}$ such that $\vx^*$ and
    $(\lambda^*,\mu^*)$ satisfy the subdifferential KKT conditions above.
    Conversely, any KKT point of $(\text{P})$ is a global minimizer.
\end{prop}

In other words, for convex problems satisfying Slater's condition, the KKT conditions provide an exact characterization of optimal solutions.

\subsection{Auxiliary Results from the Literature}
We restate the external lemmas and assumptions needed in our proofs, in formulations specialized to our notation. Proofs are omitted and can be found in the citetd references.

\begin{lem}[Existence of a restricted form solution to (\ref{equ::min_nuclear_norm}), \citet{huang2025generalization} Lemma 3] 
\label{lem::RFS1}
Suppose $\mW$ is the solution to the optimization problem (\ref{equ::min_nuclear_norm}) with identity task. There exists a solution with parameter $a_1,  a_2, b_1, b_2, c_1, c_2, d_1, d_2, e, f, g, h$ such that 
\begin{equation}
\label{equ::RFS1}
\mW^{\T}= \left( \begin{array}{llll}
a_1 \boldsymbol{I}_n+ a_2 \boldsymbol{E}_n & b_1 \boldsymbol{I}_n+b_2 \boldsymbol{E}_n & e \mathbf{1}_n & f \mathbf{1}_n \\
c_1 \boldsymbol{I}_n+c_2 \boldsymbol{E}_n  & d_1 \boldsymbol{I}_n+d_2 \boldsymbol{E}_n & g \mathbf{1}_n & h \mathbf{1}_n
\end{array} \right).
\end{equation}
The parameters follow the following constraints:
\begin{equation}
    \begin{aligned}
        &a_1, d_1 \ge 1, \\
        &a_1 + a_2 + e \ge c_1 + c_2 + g + 1, \\
        &d_1 + d_2 + h \ge b_1 + b_2 + f + 1.
    \end{aligned}
\end{equation}
In addition, after introducing the identity mapping, the problem gains additional constraints:
\begin{equation}
    \begin{aligned}
        &b_1 \ge 1, \\
        &b_1 +b_2 \ge d_1 + d_2 + 1.
    \end{aligned}
\end{equation}
\end{lem}

To analyze the optimal solution of optimization problem \ref{equ::min_nuclear_norm}, we need an explicit formula of the nuclear norm. After block multiplication, $\mW^{\T} \mW$ can be written as:
\begin{equation}
\label{equ::restricted_form}
\boldsymbol{W}^{\T} \boldsymbol{W}=\left(\begin{array}{ll}
C_{A 1} \boldsymbol{I}_n+C_{A 2} \boldsymbol{E}_n & C_{B 1} \boldsymbol{I}_n+C_{B 2} \boldsymbol{E}_n \\
C_{B 1} \boldsymbol{I}_n+C_{B 2} \boldsymbol{E}_n & C_{D 1} \boldsymbol{I}_n+C_{D 2} \boldsymbol{E}_n
\end{array}\right),
\end{equation}
where the coefficients are:
\begin{equation*}
\begin{aligned}
& C_{A 1}= a_1^2+ b_1^2 \\
& C_{A 2}= 2 a_1 a_2+ n a_2^2+ 2 b_1 b_2+ n b_2^2 + e^2+ f^2 \\
& C_{D 1}= c_1^2+ d_1^2 \\
& C_{D 2}=2 c_1 c_2+n c_2^2+ 2 d_1 d_2+n d_2^2 + g^2+ h^2 \\
& C_{B 1}= a_1 c_1+ b_1 d_1 \\
& C_{B 2}=a_1 c_2+a_2 c_1+n a_2 c_2+b_1 d_2+b_2 d_1+n b_2 d_2 + eg + fh.
\end{aligned}
\end{equation*}
The proof of Lemma \ref{lem::RFS1} is provided in \citet{huang2025generalization}.
Readers may also consult the proof of Lemma \ref{lem::RSF2}, as the underlying ideas are closely related.
{Through direct calculation, we have an explicit expression for the nuclear norm of $\mW$. We introduce the following notation to simplify the formula.
\begin{equation}
\label{equ::extra_notation1}
    \boldsymbol{u} = \left( a_1 + n a_2, b_1 + n b_2, \sqrt{n} e, \sqrt{n} f\right)^{\T}, \quad \boldsymbol{v} = \left( c_1 + n c_2, d_1 + n d_2, \sqrt{n} g, \sqrt{n} h\right)^{\T}
\end{equation}
and 
\begin{equation}
\label{equ::extra_notation2}
    A = \left( \begin{array}{cc}
        a_1 & b_1 \\
        c_1 & d_1
    \end{array}\right), \quad B = (\boldsymbol{u}, \boldsymbol{v})
\end{equation}
\begin{lem}[Closed-form of $\|\mW\|_*$ under the restricted form]
\label{lem:nuclear_norm}
Let $\mW$ be in the restricted form Eq. (\ref{equ::RFS1}). Then its nuclear norm admits the closed form
\begin{equation}
\label{equ::nuclear_norm}
\begin{aligned}
\|\mW\|_* &=
(n-1)\sqrt{\,C_{A1}+C_{D1}+2 \lvert a_1 d_1 - b_1 c_1\rvert\,} \\
&\ + \sqrt{\,C_{A1}+nC_{A2}+C_{D1}+nC_{D2}
+2\sqrt{(C_{A1}+nC_{A2})(C_{D1}+nC_{D2})-(C_{B1}+nC_{B2})^{2}}\,}.
\end{aligned}
\end{equation}
In particular, Eq. (\ref{equ::nuclear_norm}) can be reformulated as 
\begin{equation}
\label{equ::nuclear_norm_simple}
    \|\mW\|_* = (n-1) \| A \|_* + \|B \|_*.
\end{equation}
\end{lem}
}

\begin{proof}
{
Eq. (\ref{equ::nuclear_norm}) follows directly from the Lemma 6 in \citet{huang2025generalization}. Therefore, we focus on Eq. (\ref{equ::nuclear_norm_simple}). For the first term, we expand the expression according to the definition.
    \begin{equation}
        \sqrt{C_{A1} + C_{D1} + 2 |a_1 d_1 - b_1 c_1|} = \sqrt{a_1^2 + b_1^2 + c_1^2 + d_1^2 + 2|a_1 d_1 - b_1 c_1|}
    \end{equation}
Then, we check it is just $\| A\|_*$. Let $X \in \mathbb{R}^{n \times 2}$ and $\sigma_1, \sigma$ be singular values, we get $\| X\|_* = \sigma_1 + \sigma_2$. We use the following identity
\begin{equation}
    \sigma_1 + \sigma_2 = \sqrt{\sigma_1^2 + \sigma_2^2 + 2 \sigma_1 \sigma_2}. 
\end{equation}
Moreover, we use
\begin{equation}
    \sigma_1^2 + \sigma_2^2 = \| X\|_{\text{F}}^2, \quad \sigma_1 \sigma_2 = \sqrt{\operatorname{det} (X^{\T} X)}
\end{equation}
Substitute $A$ into above equation, we get the result. The proof for the second term is similar.
}
\end{proof}

To fully characterize the properties of the optimal solution to the optimization problem in (\ref{equ::min_nuclear_norm}), the mere existence of a solution is insufficient. To address this limitation, \citet{huang2025generalization} introduces the following assumption and establishes the uniqueness of the solution.

\begin{assump}
\label{assump::uniqueness}
    Suppose that $\mW$ is a solution to optimization problem (\ref{equ::min_nuclear_norm}) with (or without) identity mapping, but takes a different form from (\ref{equ::RFS1}) (or (\ref{equ::restricted_form2})). Then we assume that $\mathbf{1}_{2 n}^{\T} \mW^{\T} \neq \mathbf{0}_{2n+2}^{\T}$.
\end{assump}
{Based on Assumption \ref{assump::uniqueness}, the uniqueness theorem is stated as follows.}
\begin{thm}[Uniqueness, rephrased from \citet{huang2025generalization} Theorem 5]
    {
    Given Assumption \ref{assump::uniqueness}, the solution of optimization problem (\ref{equ::min_nuclear_norm}) with (or without) identity mapping has restricted form (\ref{equ::RFS1}) (or (\ref{equ::restricted_form2})).
    }
\end{thm}

\subsection{Proof for Theorem 1}
\label{subsec::proof_for_thm1}
{
We start with the following Proposition which characterize the optimal solution of optimization problem (\ref{equ::min_nuclear_norm}) with restricted form.
\begin{prop}
\label{prop::symmetry_prop}
    Suppose that $\mW$ is a optimal solution of optimization problem (\ref{equ::min_nuclear_norm}) with restricted form (\ref{equ::RFS1}). Then, $\mW$ satisfies $\mathbf{1}_{2n}^{\intercal} \mW^{\T} = \mathbf{0}_{2n+2}^{\intercal}$.
\end{prop}
\begin{proof}
    Using the notation introduced in Eqs. (\ref{equ::extra_notation1}) and (\ref{equ::extra_notation2}) and the Lemma \ref{lem:nuclear_norm}, the optimization problem can be reformulated as 
    \begin{equation}
    \begin{aligned}
        \min_{a_i, b_i, c_i, d_i, e, f, g, h} \quad & (n-1) \| A \|_* + \|B \|_* \\
        \textrm{s.t.} \quad & g_1: a_1 \ge 1,\\
                            & g_2: b_1 \ge 1,\\
                            & g_3: d_1 \ge 1,\\
                            & g_4: a_1 + a_2 + e \ge c_1 + c_2 + g + 1, \\
                            & g_5: b_1 + b_2 \ge d_1 + d_2 + 1, \\
                            & g_6: d_1 + d_2 + h \ge b_1 + b_2 + f + 1.
    \end{aligned}
    \end{equation}
    We assert that the optimal solution to the optimization problem satisfies the KKT conditions. To this end, we only need to verify the convexity and the Slater condition. 
    \\
    \textbf{Convexity.} Since the nuclear norm is convex, the $\| A\|_*$ is convex with respect to $a_1, b_1, c_1, d_1$. For the term $\| B\|_*$, it is also convex since it is a composition of convex function and affine transformation. Moreover, the optimization problem is convex since the objective function is convex and the constraints are affine. 
    \\
    \textbf{Slater condition.} We can construct a solution which is strictly feasible. Here, we give an example. Let $a_1 = 2, b_1 = 3, c_1 = 0.5, d_1 = 1.5, a_2 = 0, b_2 = 0, c_2 = 0, d_2 = 0, e = 0, g = 0, f=0, h=4$. Through direct testing, we can see that this solution is strictly feasible.
    \\
    Thus, we can use the KKT condition to characterize the optimal solution. Observing that $\| A\|_*$ is independent of $a_2, b_2, c_2, d_2, e, f, g, h$, we consider the stationary condition with respect to them since the symmetry of $B$ with respect to $\boldsymbol{u}$ and $\boldsymbol{v}$ will be helpful. Using chain rule, we get
    \begin{equation}
        \begin{aligned}
            &(a_2): 0 \in n (\partial_{\boldsymbol{u}} \| B\|_*)_1 - \lambda_4, \quad (c_2): 0 \in n (\partial_{\boldsymbol{v}} \| B\|_*)_1 + \lambda_4 \\
            &(b_2): 0 \in n (\partial_{\boldsymbol{u}} \| B\|_*)_2 - \lambda_5 + \lambda_6, \quad  (d_2): 0 \in n (\partial_{\boldsymbol{v}} \| B\|_*)_2 + \lambda_5 - \lambda_6 \\
            &(e): 0 \in (\partial_{\boldsymbol{u}} \| B\|_*)_3 - \lambda_4, \quad  (g): 0 \in (\partial_{\boldsymbol{v}} \| B\|_*)_3 + \lambda_4, \\
            &(f): 0 \in (\partial_{\boldsymbol{u}} \| B\|_*)_4 + \lambda_6, \quad  (h): 0 \in (\partial_{\boldsymbol{v}} \| B\|_*)_4 - \lambda_6, \\
        \end{aligned}
    \end{equation}
    Utilizing the symmetry, we find that there exist $\boldsymbol{s}_{\boldsymbol{u}} \in \partial_{\boldsymbol{u}} \| B \|_*$ and  $\boldsymbol{s}_{\boldsymbol{v}} \in \partial_{\boldsymbol{v}} \| B \|_*$ such that 
    \begin{equation}
    \label{equ::sym_subgradient}
        \boldsymbol{s}_{\boldsymbol{u}} + \boldsymbol{s}_{\boldsymbol{v}} = 0.
    \end{equation}
    \\
    Next, we proceed by case analysis.
    \begin{enumerate}[leftmargin=*]
        \item $\boldsymbol{u} = \boldsymbol{v} = 0$. In this case, we get $f = h = 0$ by the definition of $\boldsymbol{u}$ and $\boldsymbol{v}$. Then the inequality constraints $g_5$ and $g_6$ will be reformulated as 
        \begin{equation}
            b_1 + b_2 \ge d_1 + d_2 + 1, \quad d_1 + d_2 \ge b_1 + b_2 + 1.
        \end{equation}
        It makes a contradiction and implies that $\boldsymbol{u} = \boldsymbol{v} = 0$ is not a solution.
        \item $\boldsymbol{u} = 0$ or $\boldsymbol{v} = 0$. We just need to consider the $\boldsymbol{u} = 0$ case and the proof of $\boldsymbol{v} = 0$ case is similar. First, we recall the definition of the subderivative of nuclear norm. Let $X \in \mathbb{R}^{m \times n}$ with singular value decomposition $X = U \Sigma V^{\T}$, the subderivative of nuclear norm with respect to $X$ is 
        \begin{equation}
            \partial \| X\|_*  = \left\{ UV^{\T} + M : U^{\T} M=0, MV =0, \| M\|_2 \le 1 \right\}
        \end{equation}
        Substitute $B = (0, \boldsymbol{v})$ into the definition, we find
        \begin{equation}
            \boldsymbol{v}^{\T} M = 0, \quad M \boldsymbol{e}_2 = 0.
        \end{equation}
        It implies that 
        \begin{equation}
            \partial \| B\|_* = \left\{ (\boldsymbol{w}, \boldsymbol{v}) : \boldsymbol{w} \perp \boldsymbol{v}, \| \boldsymbol{w}\| \le 1 \right\}.
        \end{equation}
        Then the condition (\ref{equ::sym_subgradient}) implies that there exists $\boldsymbol{w}$ such $\boldsymbol{w}+ \boldsymbol{v} = 0$. It makes a contradiction since $\boldsymbol{w} \perp \boldsymbol{v}$.
        \item $\boldsymbol{u} \neq 0$ and $\boldsymbol{v} \neq 0$. We will further discuss two scenarios in this case. 
            \begin{enumerate}
                \item $\boldsymbol{u}$ and $\boldsymbol{v}$ are linearly independent. In this case, the subgradient equals the gradient because there are no singularities. By direct computation, we get 
                \begin{equation}
                \begin{aligned}
                    &\nabla_{\boldsymbol{u}} \| B\|_* = \frac{1}{\| B\|_*} \left( \boldsymbol{u} + \frac{\|\boldsymbol{v}\|_2^2 \boldsymbol{u} - \langle \boldsymbol{u},\boldsymbol{v}\rangle \boldsymbol{v}}{\sqrt{\|\boldsymbol{u}\|_2^2 \|\boldsymbol{v}\|_2^2 - \langle \boldsymbol{u},\boldsymbol{v} \rangle^2 }}\right) \\
                    &\nabla_{\boldsymbol{v}} \| B\|_* = \frac{1}{\| B\|_*} \left( \boldsymbol{v}+ \frac{\|\boldsymbol{u}\|_2^2 \boldsymbol{v} - \langle \boldsymbol{u},\boldsymbol{v}\rangle \boldsymbol{u}}{\sqrt{\|\boldsymbol{u}\|_2^2 \|\boldsymbol{v}\|_2^2 - \langle \boldsymbol{u},\boldsymbol{v} \rangle^2 }}\right)
                \end{aligned}
                \end{equation}
                Thus, the condition (\ref{equ::sym_subgradient}) implies that 
                \begin{equation}
                    \frac{1}{\| B\|_*} \left( 1 + \frac{\|\boldsymbol{v}\|_2^2 - \langle \boldsymbol{u},\boldsymbol{v} \rangle}{\sqrt{\|\boldsymbol{u}\|_2^2 \|\boldsymbol{v}\|_2^2 - \langle \boldsymbol{u},\boldsymbol{v} \rangle^2}}\right) \boldsymbol{u} + \frac{1}{\| B\|_*} \left( 1 + \frac{\|\boldsymbol{u}\|_2^2 - \langle \boldsymbol{u},\boldsymbol{v} \rangle}{\sqrt{\|\boldsymbol{u}\|_2^2 \|\boldsymbol{v}\|_2^2 - \langle \boldsymbol{u},\boldsymbol{v} \rangle^2}}\right) \boldsymbol{v} = 0.
                \end{equation}
                Since $\boldsymbol{u}$ and $\boldsymbol{v}$ are linearly independent, the coefficient must be zero. It implies
                \begin{equation}
                \label{equ::inver_CS}
                    \sqrt{\|\boldsymbol{u}\|_2^2 \|\boldsymbol{v}\|_2^2 - \langle \boldsymbol{u},\boldsymbol{v} \rangle^2} = \langle \boldsymbol{u},\boldsymbol{v} \rangle  - \| \boldsymbol{v}\|^2 \quad \text{and} \quad \sqrt{\|\boldsymbol{u}\|_2^2 \|\boldsymbol{v}\|_2^2 - \langle \boldsymbol{u},\boldsymbol{v} \rangle^2} = \langle \boldsymbol{u},\boldsymbol{v} \rangle  - \| \boldsymbol{u}\|^2
                \end{equation}
                which implies $\| \boldsymbol{u} \|_2^2 = \|\boldsymbol{v}\|_2^2$. Then by Cauchy-Schwarz inequality, we get $|\langle \boldsymbol{u}, \boldsymbol{v} \rangle| \le \|\boldsymbol{u}\| \|\boldsymbol{v}\|$. However, we find $\langle \boldsymbol{u},\boldsymbol{v}\rangle \ge \| \boldsymbol{u}\| \| \boldsymbol{v}\|$ by Eq. (\ref{equ::inver_CS}) and $\| \boldsymbol{u}\| = \| \boldsymbol{v} \|$. As a result, we find $\boldsymbol{u}$ is parallel to $\boldsymbol{v}$ which makes a contradiction.
                \item $\boldsymbol{u}$ and $\boldsymbol{v}$ are linearly dependent. In this case, we prove $\boldsymbol{u} = -\boldsymbol{v}$ which finishes the proof. Let $\boldsymbol{v} = \lambda \boldsymbol{u}$ and the problem turns into whether $\lambda = -1$. By the definition of the subgradient of the nuclear norm, we find
                \begin{equation}
                    \partial \|B \|_* = \left\{ \boldsymbol{u} (1, \lambda) + M : \boldsymbol{u}^{\T} M = 0, M (1, \lambda)^{\T} = 0, \| M\|_2 \le 1 \right\}.
                \end{equation}
                Let $M = (\boldsymbol{w}_1, \boldsymbol{w}_2)$, we will find
                \begin{equation}
                    \boldsymbol{u}^{\T} \boldsymbol{w}_1 = 0, \quad \boldsymbol{u}^{\T} \boldsymbol{w}_2 = 0
                \end{equation}
                Based on condition (\ref{equ::sym_subgradient}), we have
                \begin{equation}
                    \boldsymbol{u} (1+ \lambda) + (\boldsymbol{w}_1 + \boldsymbol{w}_2) = 0.
                \end{equation}
                Since $\boldsymbol{u}$ is perpendicular $\boldsymbol{w}_1, \boldsymbol{w}_2$, we get $1+ \lambda = 0$ and finish the proof.
            \end{enumerate}
    \end{enumerate}
\end{proof}
}
Using Proposition~\ref{prop::symmetry_prop}, the nuclear norm minimization problem (\ref{equ::min_nuclear_norm}) can be simplified to a more tractable form. The following lemma provides this equivalent formulation.

\begin{lem}
\label{lem::svm_simplified_form}
The  optimization (\ref{equ::min_nuclear_norm}) problem is equivalent to the following optimization problem.
\begin{equation}
\label{equ::simple_optimization1}
\begin{aligned}
    \min_{a_i, b_i, c_i, d_i, e, f, g, h} \quad & (n-1) \sqrt{M_1} + \sqrt{2 M_2} \\
    \textrm{s.t.} \quad & a_1, b_1, d_1 \ge 1, \\
    & a_1 + a_2 + e \ge c_1 + c_2 + g + 1, \\
    & b_1 + b_2 \ge d_1 + d_2 + 1, \\
    & d_1 + d_2 + h \ge b_1 + b_2 + f + 1, \\
    & a_1 + c_1 = -n(a_2 + c_2), \quad b_1 + d_1 = -n(b_2 + d_2), \\
    & e = -g, \quad f = -h,
\end{aligned}
\end{equation}
where the terms $M_1$ and $M_2$ are defined as
\begin{align*}
    M_1 &= a_1^2 + b_1^2 + c_1^2 + d_1^2 + 2|a_1 d_1 - b_1 c_1|, \\
    M_2 &= (a_1 + n a_2)^2 + (b_1 + n b_2)^2 + n e^2 + n f^2.
\end{align*}
\end{lem}

\begin{proof}
The proof consists of two main steps. First, we establish several key identities among the problem's coefficients that arise from Proposition \ref{prop::symmetry_prop}. Second, we substitute these identities into the nuclear norm expression from (\ref{equ::nuclear_norm}) to derive the simplified objective function.

\textbf{Step 1: Deriving Identities from Proposition \ref{prop::symmetry_prop}.}
The four equality constraints presented in the lemma statement are a direct consequence of the structural properties imposed by Proposition \ref{prop::symmetry_prop}. Their derivation involves straightforward algebraic manipulation and is omitted for brevity.

These equalities lead to two crucial identities for the aggregated coefficients. First, we show that $C_{A1} + n C_{A2} = C_{D1} + n C_{D2}$. By expanding the terms, we have:
\begin{align*}
    C_{A1} + n C_{A2} &= a_1^2 + b_1^2 + n \left( 2 a_1 a_2 + n a_2^2 + 2 b_1 b_2 + n b_2^2 + e^2 + f^2\right) \\
    &= (a_1 + n a_2)^2 + (b_1 + n b_2)^2 + n e^2 + n f^2.
\end{align*}
Using the equality constraints like $a_1+c_1 = -n(a_2+c_2)$, the expression above is equivalent to $(c_1 + n c_2)^2 + (d_1 + n d_2)^2 + n g^2 + n h^2$, which is precisely the expansion of $C_{D1} + n C_{D2}$.

Second, we analyze the cross-term $C_{B1} + n C_{B2}$:
\begin{align*}
    C_{B1} + n C_{B2} &= a_1 c_1 + b_1 d_1 + n (a_1 c_2 + a_2 c_1 + n a_2 c_2 + b_1 d_2 + b_2 d_1 + n b_2 d_2 + eg + fh) \\
    &= (a_1 + n a_2) (c_1 + n c_2) + (b_1 + n b_2) (d_1 + n d_2) + n eg + n fh.
\end{align*}
Applying the equality constraints again, this simplifies to:
\begin{align*}
    C_{B1} + n C_{B2} &= -(a_1 + n a_2)^2 - (b_1 + n b_2)^2 - n e^2 - n f^2 = -M_2.
\end{align*}

\textbf{Step 2: Simplifying the Nuclear Norm.}
The second term of original objective function can be simplified based on extra equality constraints. We introduce coefficients $A = C_{A1} + n C_{A2}$, $D = C_{D1} + n C_{D2}$, and $B = C_{B1} + n C_{B2}$. From Step 1, we have established that $A=D=M_2$ and $B = -M_2$.

Consequently, the discriminant term $AD - B^2$ becomes:
$$ AD - B^2 = (M_2)(M_2) - (-M_2)^2 = M_2^2 - M_2^2 = 0. $$
When the discriminant is zero, the original nuclear norm expression (likely involving square roots of eigenvalues) simplifies significantly, yielding the objective function stated in (\ref{equ::simple_optimization1}). The remaining constraints are carried over directly, which completes the proof.
\end{proof}

We restate the optimization problem after reformulate and label each constraint to prepare for the optimal solution later.
\begin{equation}
\label{equ::simple_optimization2}
\begin{aligned}
    \min_{a_i, b_i, c_i, d_i, e, f, g, h} \quad & (n-1) \sqrt{a_1^2 + b_1^2 + c_1^2 + d_1^2 + 2 |a_1 d_1 - b_1 c_1|} \\
    &+ \sqrt{2} \sqrt{ (a_1 + n a_2)^2 + (b_1 + n b_2)^2 + n e^2 + n f^2}, \\
\end{aligned}
\end{equation}
The inequality constraints are
\begin{equation*}
\begin{aligned}
& g_1: a_1-1 \geq 0, \\
& g_2: a_1+a_2+2 e-c_1-c_2-1 \geq 0, \\
& g_3: b_1-1 \geq 0, \\
& g_4: b_1+b_2-d_1-d_2-1 \geq 0, \\
& g_5: d_1-1 \geq 0, \\
& g_6: d_1+d_2-b_1-b_2-2 f-1 \geq 0.
\end{aligned}
\end{equation*}
The equality constraints are
\begin{equation*}
    \begin{aligned}
        &h_1: a_1 + c_1 + n (a_2+ c_2) = 0, \\
        &h_2: b_1 + d_1 + n (b_2+ d_2) = 0.
    \end{aligned}
\end{equation*}

{To use KKT condition to characterize the solution, we can check that the optimization problem (\ref{equ::simple_optimization2}) is a convex problem and satisfies the Slater condition which is similar to the proof in Proposition \ref{prop::symmetry_prop}. Specific details are omitted here.
}

{Before formally proving Theorem 1, we will prove several propositions that will be helpful in the proof.}

\begin{prop}
\label{prop::identity1}
    Every optimal solution of reformulated optimization problem (\ref{equ::simple_optimization2}) satisfy:
    \begin{equation}
        a_1 + n a_2 = e,\quad c_1 + n c_2 = -e
    \end{equation}
    {in which $e \ge 0$.}
\end{prop}
\begin{proof}

Using stationarity property for parameter $c_2, a_2$ and $e$ and the fact $\boldsymbol{u} \neq 0$, we find that the optimal solution satisfies
\begin{align}
    &c_2\text{: } \quad  0 + \lambda_2 - \mu_1 n = 0 \quad \Rightarrow \quad \lambda_2 = \mu_1 n. \label{eq:S1} \tag{S1} \\
    &a_2\text{: } \quad  \sqrt{2}\frac{n(a_1 + n a_2)}{\sqrt{M_2}} - \lambda_2 - \mu_1 n = 0. \label{eq:S2} \tag{S2} \\
    &e\text{: } \quad  \sqrt{2}\frac{ne}{\sqrt{M_2}} - 2\lambda_2  = 0. \label{eq:S3} \tag{S3}
\end{align}
Substitute Eq. (\ref{eq:S1}) into Eqs. (\ref{eq:S2}) and (\ref{eq:S3}) respectively, we get 
\begin{equation*}
    a_1 + n a_2 = e.
\end{equation*}
Another equality comes from the equality constraint $a_1 + c_1 + n (a_2+ c_2) = 0$. {Moreover, $e \ge 0$ from Eq. (\ref{eq:S3}) and $\lambda_2 \ge 0$.}
\end{proof}

\begin{prop}
\label{prop::two_choose_one}
For every optimal solution of reformulated optimization problem (\ref{equ::simple_optimization2}), at least one of the following two inequality constraints is tight:
    \begin{equation}
        \begin{aligned}
            &g_1: a_1 - 1 \ge 0, \\
            &g_2: a_1 + a_2 + 2e - c_1 -c_2 -1 \ge 0.
        \end{aligned}
    \end{equation}
\end{prop}

\begin{proof}
    We prove by contradiction. We show that if both two constraints are not tight, there is a feasible perturbation of this optimal point such that object value is strictly smaller.
    {
    We take 
    \begin{equation}
        \Delta a_1 = - \alpha \varepsilon, \quad \Delta b_1, \Delta b_2, \Delta d_1, \Delta d_2 = 0, \quad \Delta e, \Delta f = 0
    \end{equation}
    By proposition \ref{prop::identity1}, we take the following variables as 
    \begin{equation}
        \Delta a_2 = \frac{\alpha}{n} \varepsilon, \quad \Delta c_1 = - \gamma \varepsilon, \quad \Delta c_2 = \frac{\gamma}{n} \varepsilon.
    \end{equation}
    Next we show that by choosing appropriate parameters $\alpha$ and $\gamma$ , we can construct a feasible descent direction.
    \\
    \textbf{Step 1: Descent direction.}
    \\
    To prove that this direction is actually a descent direction, we consider the first-order approximation of the object function. First, We find that $\Delta M_2 = 0$ by our choice. Then we consider case by case:
    \begin{enumerate}[leftmargin=*]
        \item $a_1 d_1 - b_1 c_1 = 0$. In this case, we let $- \alpha \varepsilon d_1 - b_1 (-\gamma \varepsilon) = 0$ which implies $\alpha d_1 = b_1 \gamma$. As a result, we have  
        \begin{equation}
        \begin{aligned}
            \Delta M_1 &= 2 a_1 \Delta a_1 + 2 c_1 \Delta c_1 \\
            &= - 2 \alpha \varepsilon (a_1 + \frac{c_1 d_1}{b_1}) 
        \end{aligned}
        \end{equation}
            It is a descent direction since $c_1 = \frac{a_1 d_1}{b_1} > 0$ and we can take $\alpha > 0$.
        \item $a_1 d_1 - b_1 c_1 > 0$. In this case $M_1 = a_1^2 + b_1^2 + c_1^2 + d_1^2 + 2 (a_1 d_1 - b_1 c_1)$. We take $\gamma = 0$ and $\alpha$ sufficiently small and get $\Delta M_1 = -2(a_1 + d_1) \alpha \varepsilon$ which implies that it is a descent direction.
        \item $a_1 d_1 - b_1 c_1 < 0$. In this case, we have $M_1 = a_1^2 + b_1^2 + c_1^2 + d_1^2 + 2 (b_1 c_1 - a_1 d_1)$ and $c_1 > 0$. We take $\gamma$ sufficiently small and $\alpha = 0$. It is a descent direction. 
    \end{enumerate}
    }
    \textbf{Step 2: Feasible direction.}

    For inequality constraints $g_1$ and $g_2$, we can take $\varepsilon$ small enough to ensure that they can't hit the boundary. Inequality constraints $g_3$ to $g_6$ and equality constraint $h_2$ still hold since the relevant variables have not changed. Equality constraint $h_1$ still hold due to our choice for $a_1, a_2, c_1, c_2$. As a result, the construction is a feasible direction.

\end{proof}

Then, we claim that both two inequality constraints must be tight by further analysis. {Before we prove this result, we first prove the following proposition.}

\begin{prop}
\label{prop::positive_c_1}
{There exists $N > 0$ such that, for all $n > N$, every optimal solution of reformulated optimization problem (\ref{equ::simple_optimization2}) satisfy $c_1^* > 0$.}
\end{prop}
\begin{proof}
    {
    For the case of $a_1 d_1 - b_1 c_1 \le 0$, the result is trivial since $c_1 \ge \frac{a_1 d_1}{b_1} > 0$.
    And for the case of $a_1 d_1 - b_1 c_1 > 0$, we discuss case by case using Proposition \ref{prop::two_choose_one}.
    \begin{enumerate}[leftmargin=*]
        \item $g_1$ is tight, but $g_2$ is not. Based on complementary slackness of KKT condition, we know that $\lambda_2 = 0$. However, based on Eq. (\ref{eq:S1}), we get 
        \begin{equation*}
            \mu_1 = \frac{\lambda_2}{n} = 0.
        \end{equation*}
        Then, using the KKT condition with respect to $c_1$:
        \begin{equation}
            (n-1) \frac{c_1 - b_1}{\sqrt{M}_1} + \lambda_2 - \mu_1 = 0.
        \end{equation}
        It implies that $c_1 = b_1 >0$.
        \item $g_2$ is tight, but $g_1$ is not. First, we consider the perturbation of the optimal solution. To maintain the constrains, we let
        \begin{equation}
            \Delta a_1 = \varepsilon, \Delta a_2 = -\frac{\varepsilon}{n}, \Delta c_1 = \varepsilon, \Delta c_2 = - \frac{\varepsilon}{n}.
        \end{equation}
        It is a feasible direction and we consider
        \begin{equation}
            \Delta M_1 = 2 a_1 \varepsilon + 2 c_1 \varepsilon + 2 d_1 \varepsilon - 2 b_1 \varepsilon.
        \end{equation}
        It implies that 
        \begin{equation}
        \label{equ::singularity}
            a_1 + c_1 + d_1 - b_1 = 0.
        \end{equation}
        Otherwise, it will be a descent direction. Using Proposition \ref{prop::identity1}, we get
        \begin{equation*}
            a_2 = \frac{e- a_1}{n}, c_2 = \frac{-e - c_1}{n}.
        \end{equation*}
        Combine this with the constraint $g_2$, we get
        \begin{equation}
            \frac{n-1}{n} a_1 + \frac{2n +2}{n} e - \frac{n-1}{n} c_1 = 1.
        \end{equation}
        Then, we have 
        \begin{equation}
            1 \le a_1 \le \frac{n}{n-1}, 0 \le e \le \frac{1}{2n +2}, -\frac{1}{n-1} \le c_1 \le 0.
        \end{equation}
        Using Eq. (\ref{equ::singularity}), we have 
        \begin{equation}
            b_1 - d_1 = a_1 + c_1 \ge 1 - \frac{1}{n-1}
        \end{equation}
        We consider 
        \begin{equation}
        \begin{aligned}
            M_1 &= a_1^2 + b_1^2 + c_1^2 + d_1^2 + 2(a_1 d_1 - b_1 c_1) \\
                & \ge 1 + (2 - \frac{1}{n-1})^2 + 1 + 2 = 8 - O(\frac{1}{n}).
        \end{aligned}
        \end{equation}
        Thus the leading term of objective function in this case is at least $2n\sqrt{2}$. Therefore, if a specific construction can be given whose objective function is strictly smaller than the above cases, the proof is complete. The specific construction is given below. We let
        \begin{equation}
            a_1 = 1, b_1 = 2, c_1 = 0.5, d_1 = 1, e = 0.5, f = -1.
        \end{equation}
        And parameters are yet to be determined. Due to our choice, we check constraints $g_2, g_4, g_6$ and $h_1, h_2$ to determine them.
        \begin{equation}
        \begin{aligned}
            &g_2: a_2 - c_2  + 0.5 \ge 0 \\
            &g_4: b_2 - d_2 \ge 0 \\
            &g_6: d_2 - b_2 \ge 0 \\
            &h_1: 1.5 + n (a_2 + c_2) = 0 \\
            &h_2: 3 + n (b_2 + d_2) = 0.
        \end{aligned}
        \end{equation}
        Then we take 
        \begin{equation}
            a_2 = c_2 = - \frac{3}{4n}, b_2 = d_2 = -\frac{3}{2n}
        \end{equation}
        Then we compute the objective function
        \begin{equation}
            M_1 = 6.25, \quad M_2 =0.75 + 1.25 n
        \end{equation}
        As a result, the objective function
        \begin{equation}
            (n-1) \sqrt{6.25} + \sqrt{2} \sqrt{3.75 + 1.25 n} < 2\sqrt{2} n.
        \end{equation}
        Thus, for large $n$, it makes a contradiction. 
    \item $g_1$ is tight and $g_2$ is tight. We consider the following perturbation:
    \begin{equation}
        \Delta a_1 = 0, \Delta c_1 = \varepsilon, \Delta a_2 = 0, \Delta c_2 = - \frac{\varepsilon}{n}, \Delta e = \frac{\varepsilon - \frac{1}{n} \varepsilon}{2}
    \end{equation}
    Then we will find that
    \begin{equation}
        \Delta g_1 = 0, \Delta g_2 = 2 \Delta e - \Delta c_1 - \Delta c_2 = 0, \Delta h_1 = 0.
    \end{equation}
    Then we consider the change of the objective function:
    \begin{equation}
            \Delta ((n-1) \sqrt{M_1}) = (n-1) \frac{c_1 - b_1}{\sqrt{M_1}} \varepsilon \le - (n-1) \frac{1}{\sqrt{M_1}} \varepsilon \le - \frac{n-1}{2\sqrt{2}} \varepsilon
    \end{equation}
    Here we use $\sqrt{M_1} \le 2 \sqrt{2}$. If not, we can construct the solution with smaller leading term like the proof in case $2$ which implies that the solution is not optimal. Then, we consider the second term:
    \begin{equation}
        \Delta (\sqrt{2} \|u\|_2) = \sqrt{2} \frac{ne}{\| u\|} \Delta e
    \end{equation}
    Using inequality constraint $g_4$ and $g_6$, we get $f \le -1$ which implies $\| u\| \ge \sqrt{n}$. Moreover, we get $e \le \frac{1}{2(n+1)}$ because of $c_1 \le 0$ and $g_1, g_2$ are tight. As a result, we have 
    \begin{equation}
        \Delta (\sqrt{2} \| u \|) \le \frac{\sqrt{2}}{4} \frac{n-1}{(n+1)\sqrt{n}} \varepsilon
    \end{equation}
    It implies that 
    \begin{equation}
        \Delta ((n-1) \sqrt{M_1} + \sqrt{2} \| u \|) \le  - \frac{n-1}{2\sqrt{2}} \varepsilon + O (n^{-\frac{1}{2}}) \varepsilon < 0.
    \end{equation}
    Thus, we have finished the proof.
    \end{enumerate}
    }
\end{proof}

\begin{corollary}
\label{coro::tight_g_2}
{For all optimal solution of reformulated optimization problem, the inequality constraint $g_2$ must be tight.
}
\end{corollary}

\begin{proof}
    {According to Proposition \ref{prop::two_choose_one}, we consider the case $g_1$ is tight but $g_2$ is not.}
    Based on complementary slackness of KKT condition, we know that $\lambda_2 = 0$. However, based on Eq. (\ref{eq:S1}), we get 
    \begin{equation*}
        \mu_1 = \frac{\lambda_2}{n} = 0.
    \end{equation*}
    Furthermore, due to Eq. (\ref{eq:S2}), we get 
    \begin{equation*}
        e = \frac{n \sqrt{2}}{\sqrt{B}} ( \lambda_2 + \mu_1 n ) = 0
    \end{equation*}
    However, it makes an contradiction since the following constraints are violated due to {Proposition \ref{prop::positive_c_1}}:
    \begin{equation}
        a_1 + a_2 + 2e - c_1 - c_2 - 1 = 1 - \frac{1}{n} - c_1 + \frac{1}{n} c_1 -1 < 0.
    \end{equation}
\end{proof}

we now give the proof of Theorem \ref{thm:pos-ood}.

\begin{proof}[Proof of Theorem 1]
    Thanks to {Proposition \ref{prop::positive_c_1}}, to prove $q(X,y) >0$ for all OOD query $(X,y) = ((a_i, r_1, r_2),c_i)$, we just need to show that 
    \begin{equation}
    \label{equ::good_OOD_cond}
        c_1 + c_2 + g + h > a_1 + a_2 + e + f.
    \end{equation}
    This is because
    \begin{equation}
    \begin{aligned}
        s_{(X,y), b_j} &= c_1 + c_2 + g + h - \max \{ a_1, 0\} - a_2 - e - f, \quad \forall j \in [N]\\ 
        s_{(X,y), c_j} &= c_1, \quad \forall j \neq i.
    \end{aligned}
    \end{equation}
    {Here, $c_1 > 0$ by Proposition \ref{prop::positive_c_1}. So we only need to prove $s_{(X,y), b_j} > 0$.}
    Using the constraints $e + g =0$ and $f + h = 0$, inequality (\ref{equ::good_OOD_cond}) can be reformulated as 
    \begin{equation}
    \label{equ::margin_condition}
        c_1 + c_2 - (a_1 + a_2) > 2e + 2f.
    \end{equation}
    Utilizing Corollary \ref{coro::tight_g_2}, The left side of the inequality is simplified to  
    \begin{equation}
        \begin{aligned}
            c_1 + c_2 - (a_1 + a_2) & = a_1 + a_2 + 2e - 1 - (a_1 + a_2) \\
            & = 2e -1.
        \end{aligned}
    \end{equation}
    As a result, inequality (\ref{equ::margin_condition}) holds if and only if
    \begin{equation}
        f < -\frac{1}{2}.
    \end{equation}
    However, we get a better upper bound by combining inequality constraints $g_4$ with $g_6$ 
    \begin{equation}
        \begin{aligned}
        b_1 + b_2  &\ge d_1 + d_2 + 1 \\
            & \ge b_1 + b_2 + 2f + 2,
        \end{aligned}
    \end{equation}
    which implies that 
    \begin{equation}
        f \le -1.
    \end{equation}
    As a result, inequality (\ref{equ::margin_condition}) holds which implies $q(X,y) >0$ for all OOD query.
\end{proof}

\subsection{Proof for Theorem 2}
\label{subsec::proof_for_thm2}
We begin our proof with the following lemma, which makes fuller use of the symmetry of the problem.

\begin{lem}[Existence of symmetry solution]
\label{lem::RSF2}
Suppose $\mW$ is the solution to the optimization problem \ref{equ::min_nuclear_norm} without identical task. There exists a solution with $a_1, a_2, b_1, b_2$ and $\alpha, \beta$ such that 
\begin{equation}
\label{equ::restricted_form2}
\boldsymbol{W}^{\T}= \left( \begin{array}{llll}
a_1 \boldsymbol{I}_n+ a_2 \boldsymbol{E}_n & b_1 \boldsymbol{I}_n+b_2 \boldsymbol{E}_n & \alpha \mathbf{1}_n & \beta \mathbf{1}_n \\
b_1 \boldsymbol{I}_n+b_2 \boldsymbol{E}_n  & a_1 \boldsymbol{I}_n+a_2 \boldsymbol{E}_n & \beta \mathbf{1}_n & \alpha \mathbf{1}_n
\end{array} \right).
\end{equation}
The parameters follow the following constraints:
\begin{equation}
    \begin{aligned}
        &a_1 \ge 1, \\
        &a_1 + a_2 + \alpha \ge b_1 + b_2 + \beta + 1.
    \end{aligned}
\end{equation}
\end{lem}

\begin{proof}
We show that some orthogonal transformation of $\mW^{\T}$ remains a solution to the optimization problem. Let $\sigma$ be an arbitrary permutation of ${ 1, \dots, n}$, and let $\mP_{\sigma} \in \mathbb{R}^{n \times n}$ denote the associated permutation matrix. Now, consider a permutation of the logit matrix.
    \begin{equation*}
        \sigma (\mW^{\T}) = \left( 
        \begin{array}{cc}
            \mP_{\sigma} & 0 \\
            0  & \mP_{\sigma}
        \end{array}
        \right) \mW^{\T} \operatorname{diag} \{ \mP_{\sigma},\mP_{\sigma},1,1\}.
    \end{equation*}
    We find that $\sigma(\mW^{\T})$ is still an optimal solution of the optimization problem. To verify this, we consider 
    \begin{equation*}
    \begin{aligned}
         s_{(a_i, r_1), b_j} (\sigma (\mW^{\T})) &= s_{(a_{\sigma^{-1} (i)}, r_1), b_{\sigma^{-1} (j)}} ( \mW^{\T}) \ge 1, \quad \forall j \in [n] - \{ i\}, \\
         s_{(a_i, r_1), c_j} (\sigma (\mW^{\T})) &= s_{(a_{\sigma^{-1} (i)}, r_1), c_{\sigma^{-1} (j)}} ( \mW^{\T}) \ge 1, \quad \forall j \in [n], \\
         s_{(b_i, r_1), b_j} (\sigma (\mW^{\T})) &= s_{(b_{\sigma^{-1} (i)}, r_1), b_{\sigma^{-1} (j)}} ( \mW^{\T}) \ge 1, \quad \forall j \in [n], \\
         s_{(b_i, r_1), c_j} (\sigma (\mW^{\T})) &= s_{(b_{\sigma^{-1} (i)}, r_1), c_{\sigma^{-1} (j)}} ( \mW^{\T}) \ge 1, \quad \forall j \in [n] - \{ i\}.
    \end{aligned}
    \end{equation*}
    Moreover,  $\sigma(\mW^{\T})$ is another solution since orthogonal transformation does not change the nuclear norm. Consider the average over all possible permutations:
    \begin{equation*}
            \frac{\sum_{\sigma} \sigma(\mW^{\T})}{n !} = \left( \begin{array}{llll}
    a_1 \boldsymbol{I}_n+ a_2 \boldsymbol{E}_n & b_1 \boldsymbol{I}_n+b_2 \boldsymbol{E}_n & e \mathbf{1}_n & f \mathbf{1}_n \\
    c_1 \boldsymbol{I}_n+c_2 \boldsymbol{E}_n  & d_1 \boldsymbol{I}_n+d_2 \boldsymbol{E}_n & g \mathbf{1}_n & h \mathbf{1}_n
    \end{array} \right).
    \end{equation*}
    We further exploit the symmetry and consider the following transformation
    \begin{equation*}
        \tau(\mW^{\T}) = \left( 
        \begin{array}{cc}
            0 & \mI \\
            \mI & 0
        \end{array}
        \right) \mW^{\T} \operatorname{diag} \left\{ \left( 
        \begin{array}{cc}
            0 & \mI \\
            \mI & 0
        \end{array}
        \right), \left( 
        \begin{array}{cc}
            0 & 1 \\
            1 & 0
        \end{array}
        \right) \right\}.
    \end{equation*}
    Similar verification shows that $\tau(\mW^{\T})$ is still the solution to the optimization problem. Taking the sum of $\mW^{\T}$ and $\tau(\mW^{\T})$, we finish the proof.
\end{proof}

We have $\mathbf{1}^{\T} \mW^{\T} = 0$ similar to the proof of Proposition \ref{prop::symmetry_prop} and consider the problem
\begin{equation}\label{equ::optimization_faliure}
F \;=\; (n-1)\sqrt{\,2(a_1^2+b_1^2)+2|a_1^2-b_1^2|\,}\;+\;2\sqrt{(a_1+n a_2)^2+n\alpha^2}
\end{equation}
subject to
\begin{equation}
\begin{cases}
a_1 \ge 1, \\[3pt]
a_1+a_2+\alpha \;\ge\; b_1+b_2-\alpha+1, \\[3pt]
a_1+b_1+n(a_2+b_2)=0.
\end{cases}
\end{equation}
To prove Theorem \ref{thm:neg-ood}, we just need to prove the following condition holds for the optimal solution of optimization problem (\ref{equ::optimization_faliure}):
\begin{equation}
    b_1 + b_2 < a_1 + a_2
\end{equation}

\begin{proof}[Proof of Theorem 2]

\textbf{Step 0: Structural simplification.}
Using the identity
\[
a^2+b^2+|a^2-b^2|=2\max\{a^2,b^2\},
\]
the first square root in Eq. (\ref{equ::optimization_faliure}) reduces to
\[
\sqrt{\,2(a_1^2+b_1^2)+2|a_1^2-b_1^2|\,} \;=\; 2\max\{|a_1|,|b_1|\}.
\]
Since $a_1\ge1>0$, the objective becomes
\[
F \;=\; 2(n-1)\max\{a_1,|b_1|\}+2\sqrt{X^2+n\alpha^2}, \qquad X=a_1+n a_2.
\]
Meanwhile, the linear constraint rewrites as
\[
X=-(b_1+n b_2).
\]

\textbf{Step 1. Necessary condition at optimum.}
Fixing $(a_1,a_2)$ (so $X$ is fixed), we can vary $(b_1,b_2)$ while preserving $b_1+n b_2$; this leaves $X$ and the second term unchanged. If $|b_1|>a_1$, then decreasing $|b_1|$ towards $a_1$ reduces the first term and relaxes or maintains the inequality constraint, hence cannot be optimal. Therefore,
\[
|b_1|\le a_1,
\]
and the first term simplifies to $2(n-1)a_1$.

\textbf{Step 2. KKT analysis in the strict case $|b_1|<a_1$.}
In this regime, the objective is independent of $b_1,b_2$. We write down the stationary condition for $b_1, b_2$:
\begin{equation}
    \begin{aligned}
        \lambda_2 - \mu_1 &= 0 \\
        \lambda_2 - n \mu_1 & = 0.
    \end{aligned}
\end{equation}
Thus, we get $\lambda_1 = \mu_1 = 0$. Then we consider stationary condition for $a_2$:
\begin{equation}
    2n \frac{X}{\sqrt{X^2 + n \alpha^2}} - \lambda_2 - n \mu_1 = 0. 
\end{equation}
It implies $X = 0$. Finally, we consider $a_1$:
\begin{equation}
    2(n-1) + 2 \frac{X}{\sqrt{X^2 + n \alpha^2}} -\lambda_1 - \lambda_2 - \mu_1 = 0.
\end{equation}
As a result $\lambda_1 = 2(n-1)$. Based on complementary slackness, we have $a_1 =1$. Thus,
\[
a_2=-\frac1n,\qquad b_2=-\frac{b_1}{n},\qquad b_1+n b_2=0.
\]
Consequently,
\[
a_1+a_2=1-\frac1n, \qquad b_1+b_2=(1-\tfrac1n)b_1,
\]
and since $|b_1|<1$, we obtain
\[
b_1+b_2 < 1-\frac1n = a_1+a_2.
\]

\textbf{Step 3. Boundary case $|b_1|=a_1$.} 
We first find that $a_1 = 1$. Otherwise, we can perturb $a_1, a_2, b_1, b_2$ to lower the objective function while keeping $X$ fixed. The first term of the optimization object is $2(n-1)$. Moreover, we can consider another solution
\begin{equation}
    a_1 =1, a_2 = -\frac{1}{n}, b_1 = - \frac{1}{n-1}, b_2 = \frac{1}{n(n-1)}, \alpha = 0.
\end{equation}
Direct verification shows that the constraints are satisfied and the value of the objective function is $2(n-1)$. Thus, if we can find optimal solution in this case, the second term of objective function must be zero. As a result, we have
\begin{equation}
    a_1 + n a_2  = 0, \alpha = 0.
\end{equation}
Moreover, it implies that
\begin{equation}
    a_2 = -\frac{1}{n}, b_1 + n b_2 = 0
\end{equation}

Then we consider the following cases:
\begin{enumerate}[leftmargin=*, label=(\roman*)]
    \item $b_1 = -a_1 = -1$. We have $b_2 = \frac{1}{n}$ and $a_1 + a_2 = 1 - \frac{1}{n} > b_1 + b_2$.
    \item $b_1 = a_1 = 1$. It contradicts with the constraint $a_1 + a_2 + \alpha \ge b_1 + b_2 -\alpha + 1$.
\end{enumerate}
\end{proof}

\section{Experimental Details}

\subsection{Architecture Details}
\label{app:architecture}

\paragraph{Transformer (GPT-2 style).}
We use a standard decoder-only transformer with pre-norm residual blocks.
Let $d_{\text{vocab}}$ be the vocabulary size, $d_m$ the model width, $d_k$ the
per-head query/key dimension, $H$ the number of heads, $L$ the number of layers, and $T$ the context length.
Tokens $x_{1:T}$ are embedded by a lookup matrix
$\mE\!\in\!\mathbb{R}^{d_{\text{vocab}}\times d_m}$ and summed with learned
positional embeddings. Each layer $\ell=1,\dots,L$ applies

\[
\begin{aligned}
\mZ^{(\ell)} &= \mX^{(\ell)} + \mathrm{MHA}\!\left(\mathrm{LN}\!\left(\mX^{(\ell)}\right)\right),\\
\mX^{(\ell+1)} &= \mZ^{(\ell)} + \mathrm{MLP}\!\left(\mathrm{LN}\!\left(\mZ^{(\ell)}\right)\right),
\end{aligned}
\]

where $\mathrm{MHA}$ is causal multi-head attention with $H$ heads
(queries/keys/values computed by linear maps in $\mathbb{R}^{d_m\times Hd_k}$ and
output projection in $\mathbb{R}^{Hd_k\times d_m}$), and $\mathrm{MLP}$ is a
two-layer feed-forward network with GELU activation and hidden size $4d_m$.
LayerNorm (LN) is applied in the pre-norm configuration; dropout is disabled unless stated.
The language-model head shares weights with $\mE$ (tied embeddings) and projects to logits in $\mathbb{R}^{d_{\text{vocab}}}$.
\subsection{Synthetic Dataset Setup Details}

\subsection{More real world tasks}\label{sec:other_real_world_tasks}

As a supplement, similar evaluations on the MuSiQue dataset exhibit the same subject-to-answer association observed in TWOHOPFACT. In particular, when models produce correct multi-hop predictions, their internal representations and decoding behavior consistently reflect a direct dependency between the subject and the final answer, despite differences in dataset construction and question formats. This consistency suggests that the emergence of subject-to-answer associations is not specific to a particular benchmark, but instead reflects a more general property of models that successfully perform two-hop reasoning.

\begin{figure}[!htb]
    \centering
    \includegraphics[width=0.9\linewidth]{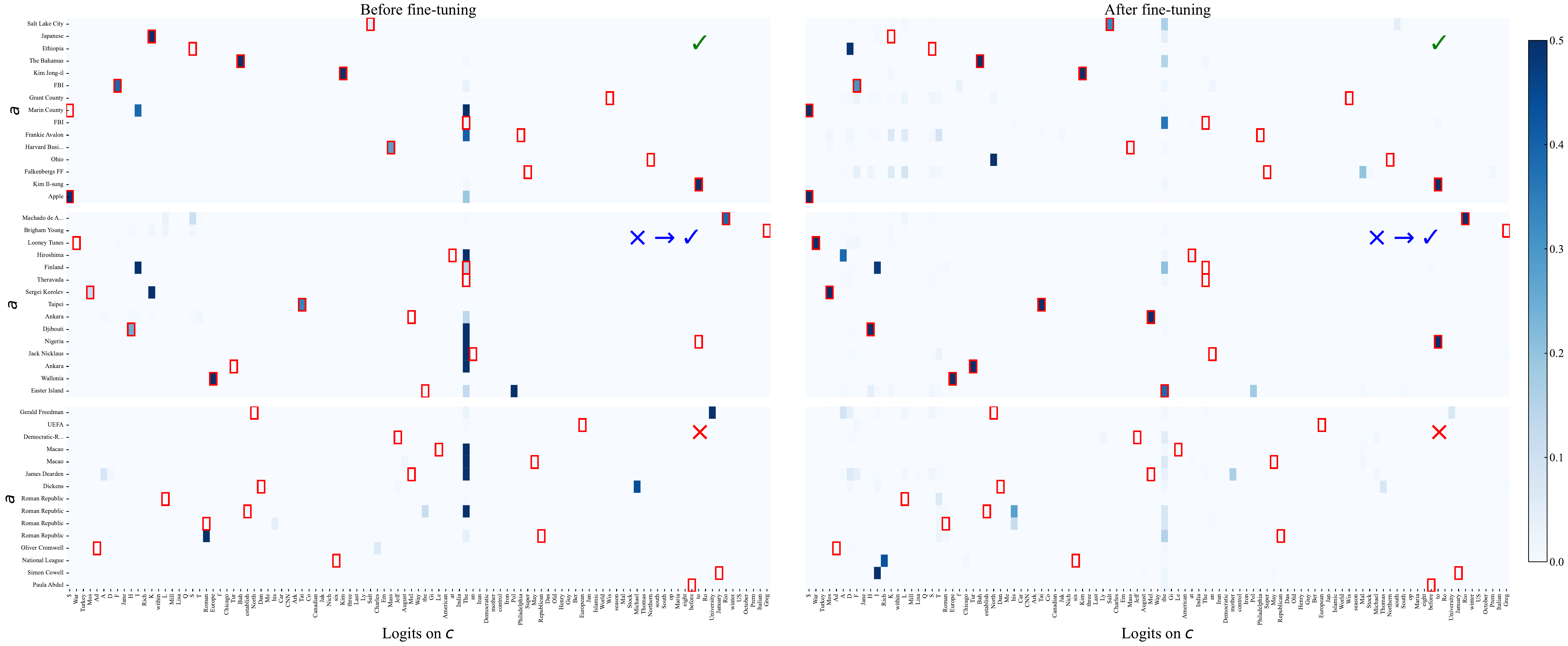}
    \caption{The probability of the model outputting the alternative city token when using the prompt corresponding to $(e_1, r_2)$ before and after fine-tuning on datasets. The red box indicates the answer to the corresponding two-hop reasoning data. The bar chart shows the change in the probability of the corresponding two-hop answer when using the prompt $(e_1, r_2)$ for different datasets. }
    \label{fig:layer_-1_novel_city_ar2_Musque-Qwen8B}
\end{figure}

\subsection{LLMs learn identity bridge in pretraining}
To further examine whether pretrained language models already encode an identity-consistent bridge mechanism, we analyze the repeat accuracy of OLMo checkpoints on a controlled probing task. In this experiment, the repeated entity is set to be the bridge token b in the TWOHOPFACT dataset. High repeat accuracy indicates that the model can reliably preserve and reproduce the bridge entity across positions, which is consistent with an identity mapping on the bridge token. Empirically, we find that OLMo checkpoints achieve consistently high repeat accuracy, suggesting that pretrained representations already support an implicit identity bridge prior to task-specific fine-tuning.

\begin{figure}[!htb]
    \centering
    \includegraphics[width=0.8\linewidth]{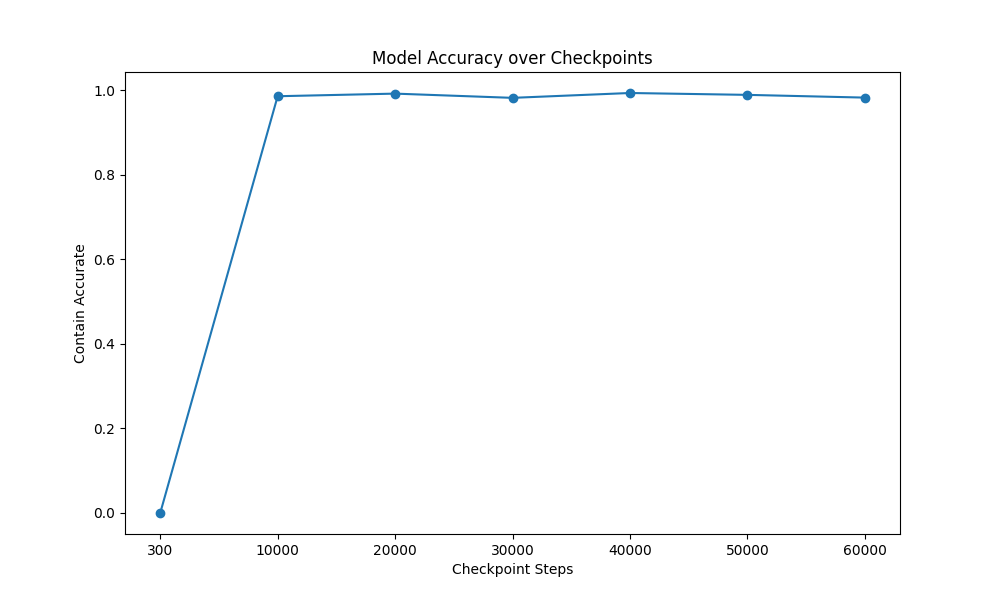}
    \caption{The repeat accuracy of Olmo checkpoints. The repeat entity are set as the bridge token $b$ of TWOHOPFACT dataset.}
    \label{fig:IDbridge_test}
\end{figure}

\subsection{Discussion about multi-hop reasoning and reverse curse}
\paragraph{Extension to 3-Hop Reasoning}\label{sec::extension_to_3hop}
Although our primary analysis focuses on two-hop reasoning, the same framework naturally extends to 3-hop reasoning with the single-layer Emb--MLP model. For simplicity, we consider the following chain of 3-hop reasoning:
\begin{equation*}
    a \xrightarrow{r_1} b\xrightarrow{r_2} c\xrightarrow{r_3}d
\end{equation*}
Based on our theoretical findings, the Emb--MLP model can learn 3-hop reasoning only if it successfully identifies the relationships from $a$ to $d$, which it can extract via the critical token $r_3$. In Fig.~\ref{fig::logits-3hop}, we present four different training setups: (a) only 1-hop pairs $(a, r_1, b), (b,r_2,c), (c, r_3, d)$; (b) 1-hop pairs with identity bridges $b \to b, c\to c$; (c) 1-hop pairs combined with two-hop pairs $(a, r_1, r_2, c), (b, r_2, r_3, d)$; and (d) a combination of 1-hop, two-hop pairs, and identity bridges.

In Fig.~\ref{fig::logits-3hop}(a) and (b), the model successfully captures the first and second hop patterns but fails to establish the 3-hop relationship. Training with two-hop data, however, strengthens the logits from $a\to c$ and from $b \to d$, which helps the model successfully complete the $a\to d$ connection, as shown in Fig.~\ref{fig::logits-3hop}(c) and (d).

\paragraph{Connection to the `Reversal Curse'.}
Some studies have observed that autoregressive large models fail to recognize the reversal of relations, such as not understanding that $a=b$ implies $b=a$. Several works have proposed transformer improvements to address this issue. In Fig.~\ref{fig::logits-3hop}(d), we observe that when the output is $c$, there is a clear activation on the diagonal of the group $\mathcal{E}_3 \to \mathcal{E}_2$ where the output is $b$. This indicates that with the help of the identity bridge, the model can establish connections such as $b \to b, c \to c, b \to c$, thereby facilitating the model's awareness of the relationship from $c$ to $b$. This suggests that the identity bridge may also offer a potential solution to the Reversal Curse problem.

For completeness, we include additional results on 3-hop reasoning with the single-layer Emb–MLP model. Fig.~\ref{fig::logits-3hop} visualizes row-wise logit patterns under different training setups, illustrating how two-hop supervision and identity bridges facilitate the formation of the final subject-to-answer association.

\begin{figure}[h]
\centering
\includegraphics[width=\linewidth]{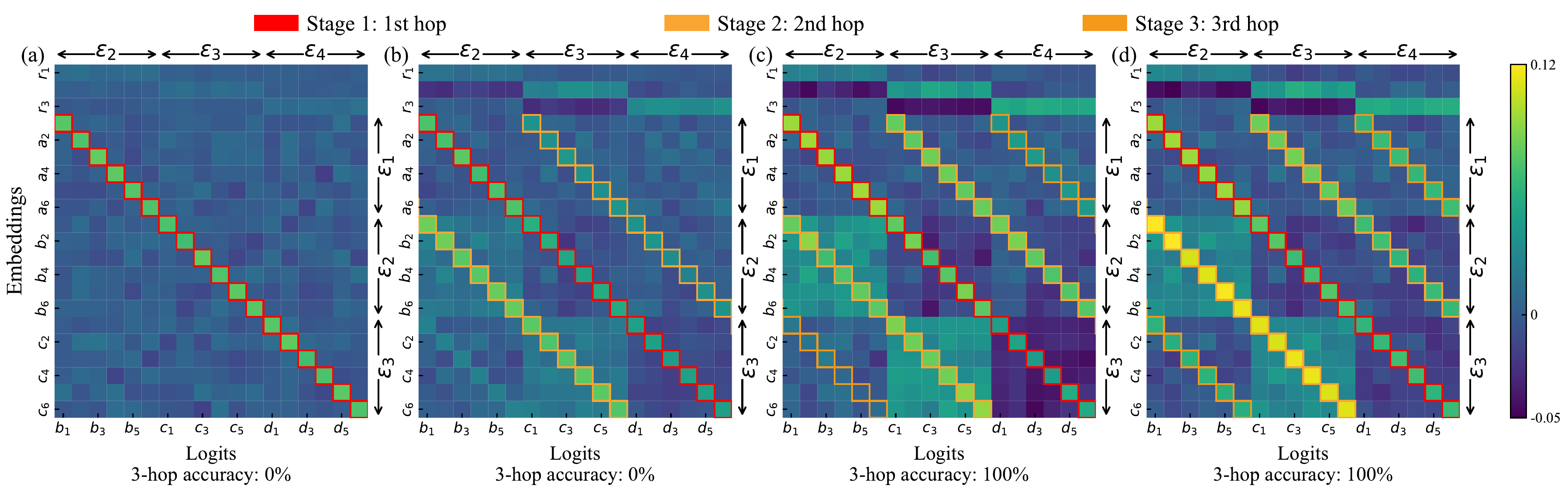}
\caption{{Row-wise logit templates in 3-hop reasoning task $(a, r_1, r_2, r_3, d)$ of Emb--MLP model in different situations: (a) train on 1-hop (b) train on 1-hop and identity bridge (c) train on 1-hop and two-hop (d) train on 1-hop, two-hop and identity bridge. two-hop and identity bridge help the model to construct the relation between $a$ to $d$ highlighted by `Stage 3'. }
}
\label{fig::logits-3hop}
\end{figure}

\section{Experiments Compute Resources}
\label{sec:resources}
The experiments were conducted on a server with the following configuration:
\begin{itemize}
    \item 48 AMD EPYC 7352 24-Core Processors, each with 512KB of cache
    \item 251GB of total system memory
    \item 8 NVIDIA GeForce RTX 4080 GPUs with 16GB of video memory each
    \item The experiments were run using Ubuntu 22.04 LTS operating system
\end{itemize}

\end{document}

%% file: math_commands.tex
\usepackage{amsmath,amsfonts,bm}

\def\eqref#1{equation~\ref{#1}}

\def\1{\bm{1}}

\def\vtheta{{\bm{\theta}}}

\def\vs{{\bm{s}}}

\def\vx{{\bm{x}}}
\def\vy{{\bm{y}}}

\def\mE{{\bm{E}}}

\def\mI{{\bm{I}}}

\def\mP{{\bm{P}}}

\def\mW{{\bm{W}}}
\def\mX{{\bm{X}}}

\def\mZ{{\bm{Z}}}

\DeclareMathAlphabet{\mathsfit}{\encodingdefault}{\sfdefault}{m}{sl}
\SetMathAlphabet{\mathsfit}{bold}{\encodingdefault}{\sfdefault}{bx}{n}

\def\sR{{\mathbb{R}}}

\usepackage{amsmath,amsfonts,bm}
\usepackage{hyperref}
\usepackage{url}
\usepackage{amsmath}  
\usepackage{amssymb}
\usepackage{amsthm} 
\usepackage{graphicx}
\usepackage{subcaption}
\usepackage{enumitem}

\newtheorem{thm}{Theorem}

\newtheorem{defi}{Definition}
\newtheorem{lem}{Lemma}
\newtheorem{prop}{Proposition}

\newtheorem{assump}{Assumption}

\newcommand{\T}{\intercal}

\def\eqref#1{equation~\ref{#1}}

\def\1{\bm{1}}

\def\vtheta{{\bm{\theta}}}

\def\vs{{\bm{s}}}

\def\vx{{\bm{x}}}
\def\vy{{\bm{y}}}

\def\mE{{\bm{E}}}

\def\mI{{\bm{I}}}

\def\mP{{\bm{P}}}

\def\mW{{\bm{W}}}
\def\mX{{\bm{X}}}

\def\mZ{{\bm{Z}}}

\DeclareMathAlphabet{\mathsfit}{\encodingdefault}{\sfdefault}{m}{sl}
\SetMathAlphabet{\mathsfit}{bold}{\encodingdefault}{\sfdefault}{bx}{n}

\def\sR{{\mathbb{R}}}




%% file: colm2026_conference.bbl
\begin{thebibliography}{34}
\providecommand{\natexlab}[1]{#1}
\providecommand{\url}[1]{\texttt{#1}}
\expandafter\ifx\csname urlstyle\endcsname\relax
  \providecommand{\doi}[1]{doi: #1}\else
  \providecommand{\doi}{doi: \begingroup \urlstyle{rm}\Url}\fi

\bibitem[Allen-Zhu \& Li(2024)Allen-Zhu and Li]{allen-zhu2025physics_3_1}
Zeyuan Allen-Zhu and Yuanzhi Li.
\newblock Physics of language models: part 3.1, knowledge storage and extraction.
\newblock In \emph{Proceedings of the 41st International Conference on Machine Learning}, ICML'24. JMLR.org, 2024.

\bibitem[Allen-Zhu \& Li(2025)Allen-Zhu and Li]{allen-zhu2025physics_3_2}
Zeyuan Allen-Zhu and Yuanzhi Li.
\newblock Physics of language models: Part 3.2, knowledge manipulation.
\newblock In \emph{The Thirteenth International Conference on Learning Representations}, 2025.
\newblock URL \url{https://openreview.net/forum?id=oDbiL9CLoS}.

\bibitem[Bai et~al.(2025)Bai, Pres, Deng, Tan, Shieber, Viégas, Wattenberg, and Lee]{bai2025canttransformerslearnmultiplication}
Xiaoyan Bai, Itamar Pres, Yuntian Deng, Chenhao Tan, Stuart Shieber, Fernanda Viégas, Martin Wattenberg, and Andrew Lee.
\newblock Why can't transformers learn multiplication? reverse-engineering reveals long-range dependency pitfalls, 2025.
\newblock URL \url{https://arxiv.org/abs/2510.00184}.

\bibitem[Balesni et~al.(2025)Balesni, Korbak, and Evans]{balesni2024two}
Mikita Balesni, Tomek Korbak, and Owain Evans.
\newblock Lessons from studying two-hop latent reasoning, 2025.
\newblock URL \url{https://arxiv.org/abs/2411.16353}.

\bibitem[Berglund et~al.(2024)Berglund, Tong, Kaufmann, Balesni, Stickland, Korbak, and Evans]{berglund2024the}
Lukas Berglund, Meg Tong, Maximilian Kaufmann, Mikita Balesni, Asa~Cooper Stickland, Tomasz Korbak, and Owain Evans.
\newblock The reversal curse: {LLM}s trained on {\textquotedblleft}a is b{\textquotedblright} fail to learn {\textquotedblleft}b is a{\textquotedblright}.
\newblock In \emph{The Twelfth International Conference on Learning Representations}, 2024.
\newblock URL \url{https://openreview.net/forum?id=GPKTIktA0k}.

\bibitem[Biran et~al.(2024)Biran, Gottesman, Yang, Geva, and Globerson]{hoptwolate}
Eden Biran, Daniela Gottesman, Sohee Yang, Mor Geva, and Amir Globerson.
\newblock Hopping too late: Exploring the limitations of large language models on multi-hop queries.
\newblock In Yaser Al-Onaizan, Mohit Bansal, and Yun-Nung Chen (eds.), \emph{Proceedings of the 2024 Conference on Empirical Methods in Natural Language Processing}, pp.\  14113--14130, Miami, Florida, USA, November 2024. Association for Computational Linguistics.
\newblock \doi{10.18653/v1/2024.emnlp-main.781}.
\newblock URL \url{https://aclanthology.org/2024.emnlp-main.781/}.

\bibitem[Ding et~al.(2024)Ding, Liu, Fu, Song, Xie, and Zhang]{shortcut1}
Mengru Ding, Hanmeng Liu, Zhizhang Fu, Jian Song, Wenbo Xie, and Yue Zhang.
\newblock Break the chain: Large language models can be shortcut reasoners.
\newblock \emph{CoRR}, abs/2406.06580, 2024.
\newblock \doi{10.48550/ARXIV.2406.06580}.
\newblock URL \url{https://doi.org/10.48550/arXiv.2406.06580}.

\bibitem[Dziri et~al.(2023)Dziri, Lu, Sclar, Li, Jian, Lin, West, Bhagavatula, Bras, and Hwang]{2023Faith}
Nouha Dziri, Ximing Lu, Melanie Sclar, Xiang~Lorraine Li, Liwei Jian, Bill~Yuchen Lin, Peter West, Chandra Bhagavatula, Ronan~Le Bras, and Jena~D. Hwang.
\newblock Faith and fate: Limits of transformers on compositionality.
\newblock \emph{ArXiv}, abs/2305.18654, 2023.

\bibitem[Heimersheim \& Nanda(2024)Heimersheim and Nanda]{heimersheim2024useinterpretactivationpatching}
Stefan Heimersheim and Neel Nanda.
\newblock How to use and interpret activation patching, 2024.
\newblock URL \url{https://arxiv.org/abs/2404.15255}.

\bibitem[Huang et~al.(2025)Huang, Zhu, Guo, Jiao, Sojoudi, Jordan, Russell, and Mei]{huang2025generalization}
Yixiao Huang, Hanlin Zhu, Tianyu Guo, Jiantao Jiao, Somayeh Sojoudi, Michael~I Jordan, Stuart Russell, and Song Mei.
\newblock Generalization or hallucination? understanding out-of-context reasoning in transformers.
\newblock \emph{arXiv preprint arXiv:2506.10887}, 2025.

\bibitem[Kazemi et~al.(2023)Kazemi, Mittal, and Ramachandran]{DBLP:journals/corr/abs-2301-11293}
Mehran Kazemi, Sid Mittal, and Deepak Ramachandran.
\newblock Understanding finetuning for factual knowledge extraction from language models.
\newblock \emph{CoRR}, abs/2301.11293, 2023.
\newblock URL \url{https://doi.org/10.48550/arXiv.2301.11293}.

\bibitem[Kojima et~al.(2022)Kojima, Gu, Reid, Matsuo, and Iwasawa]{step_by_step}
Takeshi Kojima, Shixiang~Shane Gu, Machel Reid, Yutaka Matsuo, and Yusuke Iwasawa.
\newblock Large language models are zero-shot reasoners.
\newblock \emph{Advances in neural information processing systems}, 35:\penalty0 22199--22213, 2022.

\bibitem[Lin et~al.(2025)Lin, Xie, Yuan, and Yang]{shortcut2}
Tianhe Lin, Jian Xie, Siyu Yuan, and Deqing Yang.
\newblock Implicit reasoning in transformers is reasoning through shortcuts.
\newblock In Wanxiang Che, Joyce Nabende, Ekaterina Shutova, and Mohammad~Taher Pilehvar (eds.), \emph{Findings of the Association for Computational Linguistics: ACL 2025}, pp.\  9470--9487, Vienna, Austria, July 2025. Association for Computational Linguistics.
\newblock ISBN 979-8-89176-256-5.
\newblock \doi{10.18653/v1/2025.findings-acl.493}.
\newblock URL \url{https://aclanthology.org/2025.findings-acl.493/}.

\bibitem[Liu et~al.(2026)Liu, Liu, Zhang, Zhang, Zhang, and Liu]{liu2026layer}
Xukai Liu, Ye~Liu, Jipeng Zhang, Yanghai Zhang, Kai Zhang, and Qi~Liu.
\newblock Layer-order inversion: Rethinking latent multi-hop reasoning in large language models.
\newblock \emph{arXiv preprint arXiv:2601.03542}, 2026.

\bibitem[Lv et~al.(2024)Lv, Chen, Zhang, Wang, Liu, Wen, Xie, and Yan]{lv2024interpretingkeymechanismsfactual}
Ang Lv, Yuhan Chen, Kaiyi Zhang, Yulong Wang, Lifeng Liu, Ji-Rong Wen, Jian Xie, and Rui Yan.
\newblock Interpreting key mechanisms of factual recall in transformer-based language models, 2024.
\newblock URL \url{https://arxiv.org/abs/2403.19521}.

\bibitem[Lyu \& Li(2019)Lyu and Li]{lyu2019gradient}
Kaifeng Lyu and Jian Li.
\newblock Gradient descent maximizes the margin of homogeneous neural networks.
\newblock \emph{arXiv preprint arXiv:1906.05890}, 2019.

\bibitem[Mahankali et~al.(2024)Mahankali, Hashimoto, and Ma]{1layertransformer1}
Arvind~V. Mahankali, Tatsunori Hashimoto, and Tengyu Ma.
\newblock One step of gradient descent is provably the optimal in-context learner with one layer of linear self-attention.
\newblock In \emph{The Twelfth International Conference on Learning Representations, {ICLR} 2024, Vienna, Austria, May 7-11, 2024}. OpenReview.net, 2024.
\newblock URL \url{https://openreview.net/forum?id=8p3fu56lKc}.

\bibitem[Meng et~al.(2022)Meng, Bau, Andonian, and Belinkov]{NEURIPS2022_6f1d43d5}
Kevin Meng, David Bau, Alex Andonian, and Yonatan Belinkov.
\newblock Locating and editing factual associations in gpt.
\newblock In S.~Koyejo, S.~Mohamed, A.~Agarwal, D.~Belgrave, K.~Cho, and A.~Oh (eds.), \emph{Advances in Neural Information Processing Systems}, volume~35, pp.\  17359--17372. Curran Associates, Inc., 2022.
\newblock URL \url{https://proceedings.neurips.cc/paper_files/paper/2022/file/6f1d43d5a82a37e89b0665b33bf3a182-Paper-Conference.pdf}.

\bibitem[OLMo et~al.(2024)OLMo, Walsh, Soldaini, Groeneveld, Lo, Arora, Bhagia, Gu, Huang, Jordan, et~al.]{olmo20242}
Team OLMo, Pete Walsh, Luca Soldaini, Dirk Groeneveld, Kyle Lo, Shane Arora, Akshita Bhagia, Yuling Gu, Shengyi Huang, Matt Jordan, et~al.
\newblock 2 olmo 2 furious.
\newblock \emph{arXiv preprint arXiv:2501.00656}, 2024.

\bibitem[Press et~al.(2022)Press, Zhang, Min, Schmidt, Smith, and Lewis]{press2022measuring}
Ofir Press, Muru Zhang, Sewon Min, Ludwig Schmidt, Noah~A Smith, and Mike Lewis.
\newblock Measuring and narrowing the compositionality gap in language models.
\newblock \emph{arXiv preprint arXiv:2210.03350}, 2022.

\bibitem[Qi et~al.(2023)Qi, Li, Hui, Wang, Li, Wu, and Laili]{qi2023investigation}
Chengwen Qi, Bowen Li, Binyuan Hui, Bailin Wang, Jinyang Li, Jinwang Wu, and Yuanjun Laili.
\newblock An investigation of llms' inefficacy in understanding converse relations.
\newblock \emph{arXiv preprint arXiv:2310.05163}, 2023.

\bibitem[Recht et~al.(2010)Recht, Fazel, and Parrilo]{recht2010guaranteed}
Benjamin Recht, Maryam Fazel, and Pablo~A Parrilo.
\newblock Guaranteed minimum-rank solutions of linear matrix equations via nuclear norm minimization.
\newblock \emph{SIAM review}, 52\penalty0 (3):\penalty0 471--501, 2010.

\bibitem[Trivedi et~al.(2022)Trivedi, Balasubramanian, Khot, and Sabharwal]{trivedi2021musique}
Harsh Trivedi, Niranjan Balasubramanian, Tushar Khot, and Ashish Sabharwal.
\newblock {M}u{S}i{Q}ue: Multihop questions via single-hop question composition.
\newblock \emph{Transactions of the Association for Computational Linguistics}, 2022.

\bibitem[Wang et~al.(2024)Wang, Yue, Su, and Sun]{2024Grokked}
Boshi Wang, Xiang Yue, Yu~Su, and Huan Sun.
\newblock Grokking of implicit reasoning in transformers: A mechanistic journey to the edge of generalization.
\newblock In \emph{The Thirty-eighth Annual Conference on Neural Information Processing Systems}, 2024.
\newblock URL \url{https://openreview.net/forum?id=D4QgSWxiOb}.

\bibitem[Wei et~al.(2022)Wei, Wang, Schuurmans, Bosma, Ichter, Xia, Chi, Le, and Zhou]{wei2022cot}
Jason Wei, Xuezhi Wang, Dale Schuurmans, Maarten Bosma, Brian Ichter, Fei Xia, Ed~H. Chi, Quoc~V. Le, and Denny Zhou.
\newblock Chain-of-thought prompting elicits reasoning in large language models.
\newblock In \emph{NeurIPS}, 2022.
\newblock URL \url{http://papers.nips.cc/paper_files/paper/2022/hash/9d5609613524ecf4f15af0f7b31abca4-Abstract-Conference.html}.

\bibitem[Xu et~al.(2024)Xu, Shi, and Liang]{xu2024do}
Zhuoyan Xu, Zhenmei Shi, and Yingyu Liang.
\newblock Do large language models have compositional ability? an investigation into limitations and scalability.
\newblock In \emph{ICLR 2024 Workshop on Mathematical and Empirical Understanding of Foundation Models}, 2024.
\newblock URL \url{https://openreview.net/forum?id=4XPeF0SbJs}.

\bibitem[Yang et~al.(2024)Yang, Gribovskaya, Kassner, Geva, and Riedel]{2024Do}
Sohee Yang, Elena Gribovskaya, Nora Kassner, Mor Geva, and Sebastian Riedel.
\newblock Do large language models latently perform multi-hop reasoning?
\newblock \emph{arXiv preprint arXiv:2402.16837}, 2024.

\bibitem[Yao et~al.(2025)Yao, Zhang, and Xu]{yao2025analysisreasoningbiaslanguage}
Junjie Yao, Zhongwang Zhang, and Zhi-Qin~John Xu.
\newblock An analysis for reasoning bias of language models with small initialization, 2025.
\newblock URL \url{https://arxiv.org/abs/2502.04375}.

\bibitem[Yao et~al.(2023)Yao, Yu, Zhao, Shafran, Griffiths, Cao, and Narasimhan]{ToT}
Shunyu Yao, Dian Yu, Jeffrey Zhao, Izhak Shafran, Tom Griffiths, Yuan Cao, and Karthik Narasimhan.
\newblock Tree of thoughts: Deliberate problem solving with large language models.
\newblock \emph{Advances in neural information processing systems}, 36:\penalty0 11809--11822, 2023.

\bibitem[Ye et~al.(2025)Ye, Yao, Huang, Pan, Liu, Bai, Xin, Weichuan, Che, Hou, and Li]{2025How}
Jiaran Ye, Zijun Yao, Zhidian Huang, Liangming Pan, Jinxin Liu, Yushi Bai, Amy Xin, Liu Weichuan, Xiaoyin Che, Lei Hou, and Juanzi Li.
\newblock How do transformers learn implicit reasoning?
\newblock In \emph{The Thirty-ninth Annual Conference on Neural Information Processing Systems}, 2025.
\newblock URL \url{https://openreview.net/forum?id=19ygs48nOa}.

\bibitem[Yu(2024)]{2024DoReallyThinkStep}
Yijiong Yu.
\newblock Do llms really think step-by-step in implicit reasoning?
\newblock 2024.

\bibitem[Zhang et~al.(2024{\natexlab{a}})Zhang, Frei, and Bartlett]{1layertransformer2}
Ruiqi Zhang, Spencer Frei, and Peter~L. Bartlett.
\newblock Trained transformers learn linear models in-context.
\newblock \emph{J. Mach. Learn. Res.}, 25:\penalty0 49:1--49:55, 2024{\natexlab{a}}.
\newblock URL \url{https://jmlr.org/papers/v25/23-1042.html}.

\bibitem[Zhang et~al.(2024{\natexlab{b}})Zhang, Lin, Wang, Zhang, and Xu]{zhang2024initialization}
Zhongwang Zhang, Pengxiao Lin, Zhiwei Wang, Yaoyu Zhang, and Zhi-Qin~John Xu.
\newblock Initialization is critical to whether transformers fit composite functions by reasoning or memorizing.
\newblock In \emph{The Thirty-eighth Annual Conference on Neural Information Processing Systems}, 2024{\natexlab{b}}.
\newblock URL \url{https://openreview.net/forum?id=YOBGdVaYTS}.

\bibitem[Zhang et~al.(2025)Zhang, Lin, Wang, Zhang, and Xu]{zhang2025complexity}
Zhongwang Zhang, Pengxiao Lin, Zhiwei Wang, Yaoyu Zhang, and Zhi-Qin~John Xu.
\newblock Complexity control facilitates reasoning-based compositional generalization in transformers.
\newblock \emph{arXiv preprint arXiv:2501.08537}, 2025.

\end{thebibliography}
